\documentclass{article}

\usepackage[preprint]{arxiv_preamble}

\usepackage[utf8]{inputenc} 
\usepackage[T1]{fontenc}    
\usepackage{hyperref}       
\usepackage{url}            
\usepackage{booktabs}       
\usepackage{amsfonts}       
\usepackage{nicefrac}       
\usepackage{microtype}      
\usepackage{amsfonts}
\usepackage{graphicx}
\usepackage{mathtools}
\usepackage{amsmath,amsthm}

\newcommand{\uNO}{u_{\mathrm{NO}}}
\newcommand{\K}{\mathbb{K}}
\newcommand{\R}{\mathbb{R}}
\usepackage{algorithm}
\usepackage{algorithmic}
\usepackage{amssymb}
\theoremstyle{plain}
\newtheorem{prop}{Proposition}
\usepackage{fancyhdr}
\title{Correcting Neural Operator Spectral Bias via Diffusion Posterior Sampling with Sparse Observations}

\author{%
  Niccolò Perrone \\
  Université Paris-Saclay, CentraleSupélec, CNRS, ENS Paris-Saclay \\
  Laboratoire de Mécanique Paris-Saclay UMR 9026 \\
  8-10 rue Joliot Curie, 91190 Gif-sur-Yvette, France \\
  Politecnico di Milano \\
  P.zza Leonardo da Vinci 32, 20133 Milano, Italy \\
  \texttt{niccolo.perrone@mail.polimi.it} \\
  \And
  Fanny Lehmann \\
  ETH AI Center \\
  Andreasstrasse 5, 8092 Zürich, Switzerland \\
  \texttt{fanny.lehmann@ai.ethz.ch} \\
  \And
  Stefania Fresca \\
  Department of Mechanical Engineering \\
  University of Washington, Seattle, 98195, WA, USA \\
  \texttt{sfresca@uw.edu} \\
  \And
  Filippo Gatti \\
  Université Paris-Saclay, CentraleSupélec, CNRS, ENS Paris-Saclay \\
  Laboratoire de Mécanique Paris-Saclay UMR 9026 \\
  8-10 rue Joliot Curie, 91190 Gif-sur-Yvette, France \\
  \texttt{filippo.gatti@centralesupelec.fr} \\
}
\usepackage{xcolor}

\newcommand{\FreqNODPS}{FreqNO-DPS}

\begin{document}

\maketitle

\begin{abstract}
  Neural operator surrogates (NO) can approximate partial differential equations (PDE) solutions orders of magnitude faster than numerical solvers, but they suffer from spectral bias: high-frequency content is systematically attenuated, limiting their reliability for applications that depend on fine scale structure. 
  In many settings, sparse sensor measurements of the true field are also available, offering pointwise accuracy without spectral distortion but covering only a small fraction of the domain.
  We address the spectral bias of neural operators by treating their predictions as auxiliary observations in a diffusion posterior sampling framework. Our method, \textbf{FreqNO-DPS}\footnote{Code: \url{https://github.com/niccoloperrone/FreqNO-DPS}}, combines an unconditional score-based diffusion prior, trained on high-fidelity simulations, with diffusion posterior sampling (DPS) conditioned on sparse observations and guided by a frozen neural operator.
  Na\"ive integration of the surrogate reintroduces its spectral bias into the prediction; we resolve this by deriving a closed-form, spectrally shaped guidance score that weights the surrogate contribution according to its frequency-dependent accuracy and requires no backpropagation through the denoiser.
  A distribution-free analysis bounds the approximation error across the frequency--diffusion-time plane and shows that the frequency dependence of the guidance is preserved regardless of distributional
  assumptions.
  On three-dimensional elastic wavefield prediction at $5\%$ and $2\%$ sensor coverage, the method achieves near-zero spectral bias across all frequency bands, where both the deterministic surrogate and sensor-only posterior sampling exhibit systematic high-frequency attenuation.
  Isotropic surrogate guidance, the natural baseline, improves pointwise accuracy but carries the spectral bias into the posterior nearly intact, confirming that frequency-dependent calibration is essential rather than merely beneficial.
  The framework requires only paired surrogate/reference data for calibration and exploits no problem-specific structure beyond the residual's approximate spectral diagonality, a prerequisite that can be empirically verified for new surrogates via the coherence diagnostic we provide.
\end{abstract}

\section{Introduction}
\label{sec:intro}

Neural operators (NO)~\citep{li2020fourier,kovachki2023neural}
have emerged as a promising alternative to costly numerical solvers of Partial Differential Equations (PDE), offering speedups of several orders of magnitude across fluid mechanics, solid mechanics, wave propagation, and other domains.
However, surrogates trained under mean-squared-error objectives suffer from a well-documented \emph{spectral bias}~\citep{rahaman2019spectral,khodakarami2026spectral}: low-frequency content is reproduced faithfully while high frequencies are more difficult to learn and consequently, get attenuated.
The consequences are especially severe for hyperbolic and high-frequency problems where sharp features and small-scale oscillations are preserved or amplified over time: oversmoothing has been documented for shock-capturing in Burgers and Euler conservation laws \citep{de2024numerical, urban2025approximate}, for fine scale structure in Navier–Stokes turbulence \citep{khodakarami2025mitigating, cao2024spectral}, and for high frequencies in Helmholtz and elastic-wave dynamics \citep{zou2024deep, zou2026probabilistic}.
In seismology, for instance, attenuated high frequency leads to underestimating peak amplitudes that impact buildings resistance and missing fine-scale geological features when solving inverse problems.

In many physical systems, sparse sensor networks provide direct measurements of the field of interest at a limited number of locations. In seismology, permanent and temporary station deployments record ground motion with high fidelity but at spatial densities that cover a vanishing fraction of the domain \citep{ren2026learning}; analogous settings arise in meteorology \citep{kalnay2003atmospheric}, oceanography \citep{bolton2019applications}, and structural health monitoring \citep{farrar2007introduction}. Reconstructing the complete field from such sparse observations is a severely ill-posed inverse problem \citep{manohar2018data}: many physically plausible fields are consistent with the same measurements, and the problem is further compounded when the number of observations is orders of magnitude smaller than the spatial degrees of freedom. Neither sparse sensors alone, which lack spatial coverage, nor deterministic surrogates alone, which lack spectral accuracy, suffice for high-fidelity reconstruction. 
Sparse measurements thus serve a dual role: they anchor the reconstruction at observed locations, and they expose the surrogate's spectral bias by providing the ground truth high-frequency content that the surrogate fails to reproduce.

Several lines of work have sought to correct the spectral limitations of neural operator surrogates, but none addresses the problem we consider: correcting spectral bias of a fixed, pretrained surrogate at inference. Architecture-level remedies such as retaining more Fourier modes \citep{kong2026reducing}, factorizing the spectral convolutions \citep{tran2021factorized}, or imposing multi-scale training objectives \citep{you2024mscalefno} mitigate the bias by modifying or retraining the surrogate, which is not an option when it is established, expensive to retrain, or supplied externally. Generative remedies based on score-based diffusion models \citep{song2020score, karras2022elucidating} trained on PDE data preserve spectral content that deterministic surrogates suppress \citep{molinaro2024generative, bastek2024physics}, and conditioning a diffusion model on the surrogate output recovers high-frequency structure beyond what either component achieves alone \citep{oommen2024integrating}. Such conditional models, however, couple the diffusion prior to a specific surrogate at training time and ignore the direct field observations that are often independently available in practice.

Diffusion posterior sampling (DPS; \citealp{chung2022diffusion}) provides a principled mechanism for combining sparse measurements with a pretrained generative prior. By Bayes' rule, the posterior over fields consistent with the observations is proportional to the product of the prior, here a score-based diffusion model trained on reference simulations, and a measurement likelihood. Operationally this amounts to adding a measurement-fit gradient to the unconditional reverse-diffusion drift at each step of the generative process, with no retraining required. Additional sources of information enter on the same footing as additional likelihood terms.
DPS has recently been extended to PDE-based inverse problems. Where a neural operator appears in these works, it does so as part of the model itself, as the denoiser architecture \citep{yao2025guided} or as a learned forward map between coefficient and solution spaces \citep{lin2026decoupled}, never as a parallel observation of the field being reconstructed. To our knowledge, no existing diffusion-based method treats a frozen surrogate's prediction as a direct auxiliary observation, and none explicitly targets the surrogate's spectral bias. We close this gap. A frozen neural operator enters our posterior as an auxiliary observation alongside the sparse sensors, and the corresponding likelihood is calibrated in the Fourier domain so that the surrogate's information is admitted at the low frequencies where it is reliable, while the sensor channel takes over at the high frequencies where the surrogate's bias is most severe.

We propose \textbf{\FreqNODPS}, a framework that combines an unconditional diffusion prior trained on high-fidelity numerical simulations with diffusion posterior sampling conditioned on both sparse sensor observations and a frozen neural operator prediction. The sensor term anchors the prediction at observed locations, while the neural operator term provides global structural information across the entire domain. To prevent the surrogate's spectral bias from corrupting the posterior, we derive a \emph{spectrally shaped} guidance score by marginalizing over the unknown ground-truth field in the Fourier domain. This yields a closed-form expression that accounts for the frequency-dependent accuracy of the surrogate and involves no backpropagation through the denoiser. We demonstrate \FreqNODPS\ on 3D elastic wavefield enhancement from a Multiple-Input Fourier Neural Operator (MIFNO) surrogate. While the methodological and theoretical contributions are independent of this application, our empirical experiments are specific to this challenging setting.

Our contributions are as follows:
\begin{itemize}
    \item[(i)] We introduce \FreqNODPS, a diffusion posterior sampler that incorporates a frozen neural operator as an auxiliary observation with a closed-form, frequency-calibrated likelihood, and requires no backpropagation through the denoiser (Sec.~\ref{sec:method}).

    \item[(ii)] We establish an exact, distribution-free identity for the neural operator likelihood score (Prop.~\ref{prop:exact-score-main}) and highlight four regimes depending on the frequency content and diffusion time to show that frequency dependence of the guidance is preserved regardless of distributional assumptions.
    
    \item[(iii)] We demonstrate on 3D elastic wavefield prediction at 5\% and 2\% sensor coverage that the method achieves near-zero
    spectral bias across all frequency bands, and show via ablation that isotropic surrogate guidance reimports the spectral bias nearly intact, confirming that frequency-dependent calibration is
    essential.
\end{itemize}

\section{Related Work}
\label{sec:related}

\paragraph{Spectral bias of neural operators.}
Neural networks trained under mean-squared-error objectives exhibit
a documented bias toward low-frequency content \citep{rahaman2019spectral, xu2019frequency}, and recent work has characterized this phenomenon specifically for neural operators and physics-informed learning \citep{khodakarami2026spectral,qin2024toward}.
In Fourier Neural Operator (FNO) architectures, the spectral convolution layers truncate high wavenumbers by construction \citep{li2020fourier}, although the pointwise residual branches can in principle carry high-frequency content past this truncation. In practice, predictions still exhibit systematic high-frequency attenuation \citep{kong2026reducing}, motivating dedicated remediation strategies.
Existing remediation strategies modify the surrogate itself: retaining more Fourier modes through multistage training \citep{kong2026reducing}, factorizing the spectral convolutions to support larger mode counts \citep{tran2021factorized}, or imposing multi-scale training objectives that explicitly target oscillatory function spaces \citep{you2024mscalefno}.
These approaches require modifying or retraining the surrogate. We take a complementary route: leaving a pretrained surrogate frozen and correcting its spectral bias at inference through a generative prior conditioned on sparse observations.

\paragraph{Diffusion models for PDE fields and spectral recovery.}
Score-based diffusion models~\citep{song2020score,karras2022elucidating} have recently been applied to PDE-governed fields. 
\citet{lippe2308pde} demonstrate that diffusion-inspired iterative
refinement recovers frequency components that standard neural PDE
solvers neglect, establishing a direct link between denoising
objectives and spectral fidelity. \citet{molinaro2024generative}
and \citet{bastek2024physics} show more broadly that diffusion
models trained on PDE data preserve the spectral content that
deterministic surrogates suppress. Most directly relevant,
\citet{oommen2024integrating} demonstrate that conditioning a diffusion model on neural operator predictions recovers high-frequency turbulent structures beyond what either component achieves alone, though the diffusion model is coupled to the surrogate at training time and no direct observations enter. Closest to our application setting, \citet{perrone2025integrating} bring this conditional-diffusion approach to synthetic earthquake ground motion, improving the spectral representation of the synthetic ground motion, but additionally operate on individual stations and so do not reconstruct a spatially coherent surface wavefield. Our framework departs on all three counts: an unconditional prior decoupled from any specific surrogate, sparse sensor observations admitted as a parallel channel, and reconstruction of the full 3D surface field.

\paragraph{Sparse sensor inverse problems with diffusion guidance.}
Diffusion posterior sampling~(DPS;~\citealp{chung2022diffusion})
enables zero-shot conditioning of a pretrained diffusion model on noisy linear measurements by approximating the likelihood score via the Tweedie denoised estimate~\citep{efron2011tweedie}.
Several recent works extend this framework to PDE-based inverse problems.
\citet{AmorosTrepat_MedranoNavarro_Liu_Guastoni_Thuerey_2025}
reconstruct turbulent flow fields from sparse data using a
masked-diffusion sampling strategy that overwrites the denoised
estimate at sensor locations with a smoothed interpolation of the
true observations, enforcing the measurements as a hard constraint
without backpropagation through the denoiser. CoNFiLD~\citep{liu2025confild} trains an unconditional latent diffusion model for 3D spatiotemporal turbulence and performs zero-shot sparse-sensor reconstruction via Bayesian conditional sampling. Both methods address sparse-observation reconstruction with a diffusion prior alone; within our experimental setup, the sensor-only DPS baseline of Sec.~\ref{sec:quantitative} plays the same structural role (sparse-sensor reconstruction without a neural operator channel),
differing in the specific guidance mechanism. FunDPS~\citep{yao2025guided} instead uses a neural operator architecture \emph{as} the diffusion denoiser, applying plug-and-play DPS guidance from sparse observations; in this formulation no auxiliary surrogate likelihood arises, so the spectral-bias question we study does not appear.
DDIS~\citep{lin2026decoupled} also pairs a diffusion prior with a neural operator, but the operator plays a different role: it serves as a learned \emph{forward physics surrogate} that bridges the coefficient space (where the diffusion prior lives) and the solution space (where observations live), enabling likelihood
evaluation for coefficient-from-solution inverse problems. In none of these methods does the neural operator enter the posterior as a direct, parallel observation of the field being reconstructed.
The contribution of the present work is to treat the surrogate in exactly this role, as an auxiliary observation channel of the
target field, and to derive a frequency-dependent likelihood calibrated against the surrogate's mode-dependent accuracy, so that its spectral bias is corrected rather than absorbed into the posterior.

\section{Data and Problem Setup}
\label{sec:data}
We consider the reconstruction of three-dimensional elastic surface wavefields from sparse spatial observations.
This section describes the forward simulation framework and dataset (Sec.~\ref{sec:sem-data}), the deterministic neural operator surrogate used as auxiliary information (Sec.~\ref{sec:mifno-surrogate}), and the sparse observation model that defines the inverse problem (Sec.~\ref{sec:obs-model}).

\subsection{3D elastic wave propagation}
\label{sec:sem-data}

We consider seismic-wave propagation in a heterogeneous, linearly elastic medium and follow the framework defined in \citet{lehmann2024synthetic}. Let
$\Omega = [0,\Lambda]^3 \subset \R^3$ be a cubic domain of size $\Lambda$~=~9.6\,km and denote by $\partial\Omega_{\mathrm{top}}$
its traction-free upper surface. All other external faces are equipped with absorbing boundary
conditions to approximate a semi-infinite propagation domain.

The medium is described by spatially varying geological parameters. In the general setting,
we denote by $a:\Omega\to\R$ the field parametrizing the material properties (in our data, $a$
is the shear wave velocity field $V_S$; the remaining parameters are obtained through fixed
deterministic relationships).
A seismic event is defined by a source location $\mathbf{x}_s \in \Omega$ and a source-mechanism
parameter vector $\boldsymbol{\theta}_s$ (e.g.\ a moment-tensor parametrization).
Let $\mathbf{u}:\Omega\times[0,T]\to\R^3$ denote the displacement field. The forward model can be
written abstractly as
\begin{equation}
\label{eq:forward}
\mathcal{L}(a, \mathbf{u}) = \mathbf{f}(\mathbf{x}_s,\boldsymbol{\theta}_s),
\end{equation}
where $\mathcal{L}$ is the (heterogeneous) elastic wave operator and $\mathbf{f}$ is the source term.
Ground-truth wavefields are generated with high-fidelity Spectral Element Method (SEM, \cite{Touhami_et_al_2022}) simulations and gathered in the HEMEW\textsuperscript{S}-3D, which is directly used in the present work \citep{lehmann2024synthetic}.

To reduce storage, we retain only surface time histories on $\partial\Omega_{\mathrm{top}}$,
namely the surface particle velocity $\dot{\mathbf{u}}\!\mid_{\partial\Omega_{\mathrm{top}}}$ sampled on a regular
surface grid over time.
For notational simplicity, in the remainder we denote these stored surface velocity fields by
\begin{equation}
\mathbf{u} \; \coloneqq \; \dot{\mathbf{u}}\!\mid_{\partial\Omega_{\mathrm{top}}}
\;\in\; \R^{C \times N_x \times N_y \times T},
\end{equation}
with $C=3$ velocity components (east--west, north--south, vertical). In our pipeline, we use a temporal sampling
$\Delta t = 0.02\,\mathrm{s}$ over a time window of $6.4\,\mathrm{s}$, hence $T = 320$ time steps (with $N_x=N_y=32$).

\subsection{Multiple-Input Fourier Neural Operator}
\label{sec:mifno-surrogate}

We use a pretrained Multiple-Input Fourier Neural Operator (MIFNO) \citep{lehmann2025multiple} as a deterministic surrogate for fast prediction of surface ground motion. The MIFNO learns the parametric mapping from a 3D geological input field
$a$ and source characteristics $(\mathbf{x}_s,\boldsymbol{\theta}_s)$ to the corresponding surface velocity wavefield $u$:
\begin{equation}
\label{eq:mifno-map}
\mathbf{\uNO} \;=\; G_{\phi}\!\left(a,\mathbf{x}_s,\boldsymbol{\theta}_s\right),
\qquad
G_{\phi}:\; \mathcal{A}\times\Omega\times\Theta \to \R^{C\times N_x\times N_y\times T},
\end{equation}
where $\phi$ denotes learned parameters and $\mathbf{\uNO}$ is the surrogate prediction. Architecturally, MIFNO extends
3D Fourier Neural Operators \citep{li2020fourier, tran2021factorized} to heterogeneous multi-modal inputs by processing the structured 3D geology with Fourier
operator layers while encoding the low-dimensional source parameters in a dedicated branch, before fusing both
representations to predict the surface wavefield.

The surrogate is trained in a supervised fashion to minimize discrepancy between the predicted wavefields $\mathbf{\uNO}$
and the SEM targets $\mathbf{u}$ over the training set. In our inference pipeline, the MIFNO is \emph{frozen}: it serves as
(i) a fast deterministic baseline and (ii) an auxiliary signal to stabilize diffusion-based posterior sampling in
ultra-sparse measurement regimes (Sec.~\ref{sec:method}).

Like other deterministic neural operator surrogates, MIFNO predictions can exhibit \emph{spectral bias} \citep{rahaman2019spectral}: small-scale, high-frequency fluctuations are harder to reproduce faithfully, often resulting in oversmoothed reconstructions and residual phase/spectral errors. Moreover, as a point predictor, MIFNO does not provide calibrated predictive uncertainty, which motivates the use of a generative diffusion prior and posterior sampling to recover plausible high-frequency content and produce an ensemble of reconstructions consistent with the observations.

\subsection{Sparse observation model}
\label{sec:obs-model}

We assume access to sparse measurements of the ground-truth surface velocity field
$\mathbf{u}~\in~\R^{C \times N_x \times N_y \times T}$. Observations correspond to a sparse subset of
spatial sensor locations on the $N_x \times N_y$ surface grid, while retaining the full temporal history and all
$C=3$ components at each observed location. In this work, sensor data are obtained from SEM simulations to avoid any distribution shift with the MIFNO training objective. Future work will investigate the use of real measurements to remove all dependency on the numerical solver once the MIFNO is trained.

Let $\mathcal{G} \coloneqq \{1,\dots,N_x\}\times\{1,\dots,N_y\}$
denote the spatial surface grid. Sparse measurements correspond to a subset $\mathcal{S} \subset \mathcal{G}$ of $|\mathcal{S}|$ grid locations at which the wavefield is observed; at each observed
location, the full temporal history and all $C=3$ velocity components are recorded. The observed sensor density is
$\rho \coloneqq |\mathcal{S}|/(N_xN_y)$. In the following, we
refer to regimes with very small $\rho$ (e.g., $\rho \leq 5\%$)
as \emph{ultra-sparse}. We write the associated linear restriction
operator $\mathcal{M}_{\mathcal{S}}$ that extracts the entries of
$\mathbf{u}$ at the locations in $\mathcal{S}$.
The measurement model is
\begin{equation}
\label{eq:meas-model}
\mathbf{y} \;=\; \mathcal{M}_{\mathcal{S}}(\mathbf{u})
  + \boldsymbol{\varepsilon},
\qquad
\boldsymbol{\varepsilon} \sim
  \mathcal{N}\!\left(\mathbf{0},\sigma_y^2 I\right),
\end{equation}
where $\mathbf{y} \in \mathbb{R}^{C \times |\mathcal{S}| \times T}$ stacks the observed values over all components, sensor locations, and time steps, and $\sigma_y$ controls the measurement noise level. In the experiments we report here the temporal discretization is fixed ($T=320$), and sparsity refers exclusively to the spatial sampling pattern. Details on how $\mathcal{S}$ is constructed across densities are given in Appendix~\ref{app:experimental}.

The Gaussian likelihood used by diffusion posterior sampling is
\begin{equation}
\label{eq:gauss-likelihood}
\log p(\mathbf{y} \mid \mathbf{u})
\;=\;
-\frac{1}{2\sigma_y^2}\,\big\|\mathcal{M}_{\mathcal{S}}(\mathbf{u})-\mathbf{y}\big\|_2^2 + \mathrm{const},
\end{equation}
so that gradients of the data term act only on the observed entries. This gradient reads:
\begin{equation}
\label{eq:likelihood-grad-clean}
\nabla_{\mathbf{u}}\log p(\mathbf{y}\mid \mathbf{u})
=
-\frac{1}{\sigma_y^2}\,
\mathcal{M}_{\mathcal{S}}^\dagger\!\Big(\mathcal{M}_{\mathcal{S}}(\mathbf{u})-\mathbf{y}\Big),
\end{equation}
where $\mathcal{M}_{\mathcal{S}}^\dagger$ is the adjoint operator that places values at the observed locations and fills the remaining entries with zeros.

\section{Methods}
\label{sec:method}

This section develops the posterior sampling framework in three steps: we first define the unconditional diffusion prior (Sec.~\ref{sec:prior}), then introduce standard DPS conditioning on sparse measurements (Sec.~\ref{sec:dps}), and finally derive the spectrally shaped neural operator guidance that is the main methodological contribution of this work (Sec.~\ref{sec:no-guided-dps}). A concluding subsection establishes an exact, distribution-free identity for the neural operator likelihood score and bounds the error of the moment-matched approximation (Sec.~\ref{sec:score-analysis}).
Figure~\ref{fig:architecture} provides an overview of the full pipeline.

\begin{figure}
    \centering
    \includegraphics[width=1.0\linewidth]{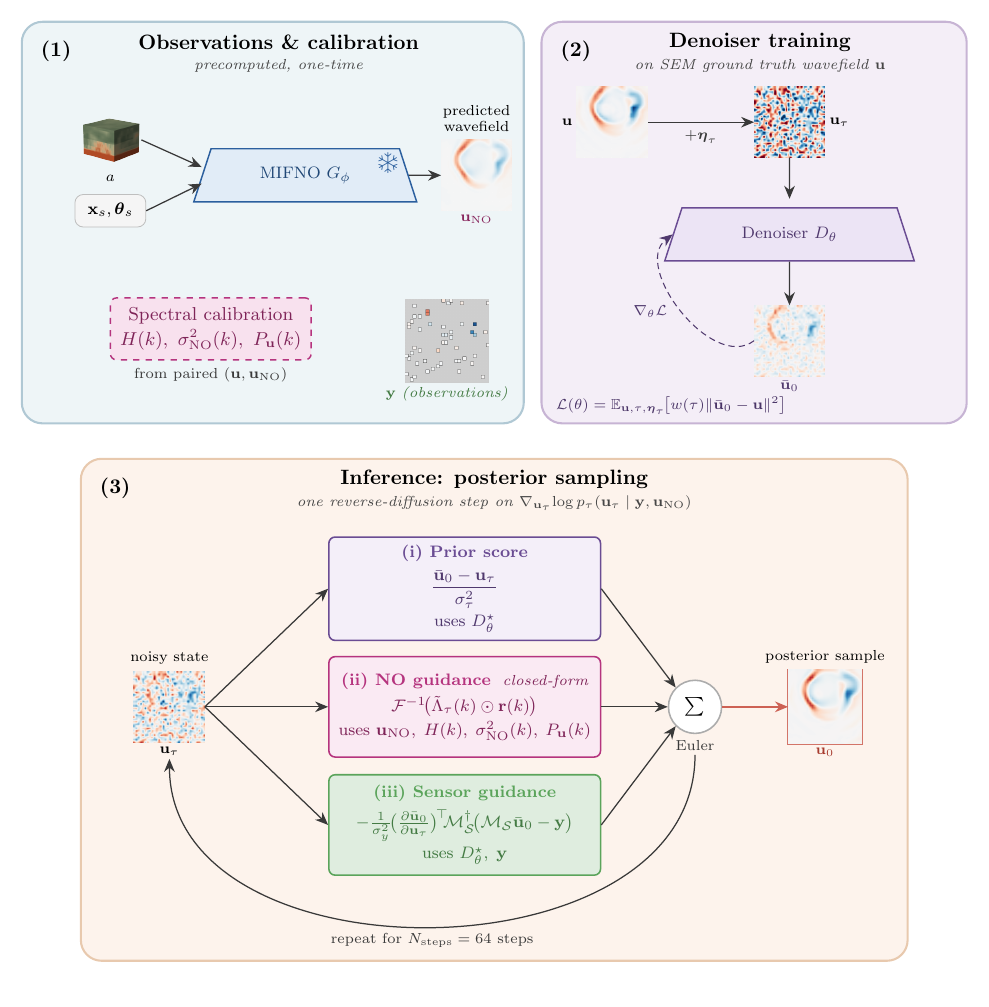}
    \caption{\textbf{Overview of \FreqNODPS.} \emph{Top-left (Observations \& calibration):} the frozen MIFNO surrogate $G_\phi$ predicts surface wavefields $\mathbf{u}_{\mathrm{NO}}$ from geology and source inputs; spectral statistics $H(k),\sigma^2_{\mathrm{NO}}(k),P_{\mathbf{u}}(k)$ are estimated once from paired $(\mathbf{u},\mathbf{u}_{\mathrm{NO}})$; sparse observations $\mathbf{y}$ are obtained from a random sensor mask on the ground truth wavefield $\mathbf{u}$. \emph{Top-right (Denoiser training):} an unconditional denoiser $D_\theta$ is trained on SEM ground truth $\mathbf{u}$ under an EDM-weighted denoising objective, yielding the frozen checkpoint $D_\theta^{\star}$ used at inference. \emph{Bottom (Inference):} at each reverse-diffusion step the posterior score combines a Tweedie prior term, a closed-form spectrally shaped NO term, and a DPS sensor term~\eqref{eq:dps-likelihood-score}; iterating $N_\mathrm{steps}{=}64$ Euler steps yields a posterior sample $\mathbf{u}_0$.}
    \label{fig:architecture}
\end{figure}

\subsection{Unconditional diffusion prior}
\label{sec:prior}

We model the distribution of SEM simulations via an \emph{unconditional} score-based diffusion model.
Under the variance-exploding (VE) schedule \citep{song2020score}, the forward process progressively corrupts a sample $\mathbf{u} \sim p_0$ by additive Gaussian noise
\begin{equation}
\label{eq:gaussian-perturb}
\mathbf{u}_\tau = \mathbf{u} + \boldsymbol{\eta}_\tau,
\qquad
\boldsymbol{\eta}_\tau \sim \mathcal{N}(\mathbf{0},\,\sigma^2_\tau I),
\end{equation}
indexed by a diffusion time $\tau \in [0, K]$ (distinct from the physical time axis of the wavefield denoted by $t$), with $\sigma_0 = 0$ and $\sigma_K=80$ large enough that $p_K \approx \mathcal{N}(\mathbf{0}, \sigma_K^2 I)$.

We train a denoiser $D_\theta(\mathbf{u}_\tau, \sigma_\tau)$ to
predict the clean sample from its noisy counterpart by minimizing
\begin{equation}
\label{eq:prior-loss}
\mathcal{L}_{\mathrm{prior}}(\theta)
=
\mathbb{E}_{\mathbf{u},\,\tau,\,\boldsymbol{\eta}_\tau}
\left[
w(\tau)\,
\big\|D_\theta(\mathbf{u}+\boldsymbol{\eta}_\tau,\,\sigma_\tau)
  - \mathbf{u}\big\|_2^2
\right],
\end{equation}
with EDM-style weighting  $w(\tau)$ \citep{karras2022elucidating}.
The denoised estimate $\bar{\mathbf{u}}_0(\mathbf{u}_\tau, \tau)\coloneqq D_\theta(\mathbf{u}_\tau, \sigma_\tau)\approx \mathbb{E}[\mathbf{u} \mid \mathbf{u}_\tau]$
provides, via Tweedie's formula \citep{efron2011tweedie}, an approximation of the score:
\begin{equation}
\label{eq:tweedie}
\nabla_{\mathbf{u}_\tau}\!\log p_\tau(\mathbf{u}_\tau)
\;\approx\;
\frac{\bar{\mathbf{u}}_0 - \mathbf{u}_\tau}{\sigma_\tau^2}.
\end{equation}
The score is the only learned ingredient required to simulate the reverse-time ODE that generates new samples from $p_0$.

\subsection{Diffusion posterior sampling with sparse measurements}
\label{sec:dps}

Given sparse sensor observations $\mathbf{y}$ from the measurement model~\eqref{eq:meas-model}, we seek samples from the posterior:
\begin{equation}
\label{eq:posterior}
p(\mathbf{u}\mid \mathbf{y})
\;\propto\;
p_\theta(\mathbf{u})\,p(\mathbf{y}\mid \mathbf{u}),
\end{equation}
where $p_\theta(\mathbf{u})$ is the unconditional diffusion prior learned in Sec.~\ref{sec:prior} and $p(\mathbf{y}\mid\mathbf{u})$ is the measurement likelihood.
The posterior score at each diffusion time decomposes as
\begin{equation}
\label{eq:posterior-score-decomp}
\nabla_{\mathbf{u}_\tau}\!\log p_\tau(\mathbf{u}_\tau \mid \mathbf{y})
\;=\;
\nabla_{\mathbf{u}_\tau}\!\log p_\tau(\mathbf{u}_\tau)
\;+\;
\nabla_{\mathbf{u}_\tau}\!\log p(\mathbf{y} \mid \mathbf{u}_\tau),
\end{equation}
so that posterior sampling augments the prior-driven reverse diffusion ($\nabla_{\mathbf{u}_\tau}\!\log p_\tau(\mathbf{u}_\tau)$) with a data-consistency drift ($\nabla_{\mathbf{u}_\tau}\!\log p(\mathbf{y} \mid \mathbf{u}_\tau)$).
Evaluating the likelihood score $\nabla_{\mathbf{u}_\tau}\!\log p(\mathbf{y} \mid \mathbf{u}_\tau)$ requires the noisy state likelihood $p(\mathbf{y} \mid \mathbf{u}_\tau)=\int p(\mathbf{y} \mid \mathbf{u}) p(\mathbf{u} \mid \mathbf{u}_\tau)d\mathbf{u}$, since the measurement model~\eqref{eq:meas-model} is defined for clean wavefields $\mathbf{u}$ only. The integral has no closed form because the conditional $p(\mathbf{u} \mid \mathbf{u}_\tau)$ depends on the data prior. DPS approximates it by collapsing $p(\mathbf{u} \mid \mathbf{u}_\tau)$ onto the denoised estimate $\bar{\mathbf{u}}_0$ as defined in~\eqref{eq:tweedie}, yielding the approximate likelihood score:
\begin{equation}
\label{eq:dps-likelihood-score}
\nabla_{\mathbf{u}_\tau}\!\log p(\mathbf{y} \mid \mathbf{u}_\tau)
\;\approx\;
-\frac{1}{\sigma_y^2}\,
\left(\frac{\partial \bar{\mathbf{u}}_0}{\partial \mathbf{u}_\tau}
\right)^{\!\top}
\mathcal{M}_{\mathcal{S}}^\dagger\!
\Big(\mathcal{M}_{\mathcal{S}}(\bar{\mathbf{u}}_0) - \mathbf{y}\Big),
\end{equation}
computed by automatic differentiation through the denoiser.
The gradient acts only on the $|\mathcal{S}|$ observed sensor locations; in ultra-sparse regimes, this constrains only a small fraction of degrees of freedom, motivating the additional neural operator guidance introduced next.

\subsection{Neural operator-guided posterior sampling}
\label{sec:no-guided-dps}
\paragraph{Augmented posterior.}
To compensate for the information deficit of DPS, we leverage the pretrained MIFNO of Sec.~\ref{sec:mifno-surrogate} by treating its frozen prediction $\mathbf{\uNO}$ as an auxiliary observation of the clean wavefield $\mathbf{u}$. The problem is thus recast as inference of the augmented posterior $p(\mathbf{u} \mid \mathbf{y}, \mathbf{\uNO})$, in which the sparse measurements act as the primary truth anchor while $\mathbf{\uNO}$ provides a global inductive bias across the entire domain, including the locations unconstrained by sensors. It is natural to assume conditional independence given the clean wavefield:
\begin{equation}
    \label{eq:cond-indep}
    p(\mathbf{y}, \mathbf{\uNO} \mid \mathbf{u})
    \;=\;
    p(\mathbf{y} \mid \mathbf{u})\,p(\mathbf{\uNO} \mid \mathbf{u})
\end{equation}
since MIFNO predictions and extraction of sensor measurements are conditionally independent
given the clean wavefield. To see this, let $\boldsymbol{\xi} = (a, \mathbf{x}_s, \boldsymbol{\theta}_s)$ denote the geological and source
parameters. The measurement $\mathbf{y} = \mathcal{M}_\mathcal{S}(\mathbf{u}) + \boldsymbol{\varepsilon}$ depends on $\mathbf{u}$ and on the instrument noise $\boldsymbol{\varepsilon}$, while the surrogate
prediction $\mathbf{u}_\mathrm{NO} = G_\phi(\boldsymbol{\xi})$ depends on $\boldsymbol{\xi}$ alone.
Conditioning on $\mathbf{u}$, the residual randomness in $\mathbf{y}$ is entirely $\boldsymbol{\varepsilon}$, and the residual randomness in $\mathbf{u}_\mathrm{NO}$ is which $\boldsymbol{\xi}$ generated $\mathbf{u}$ (the inverse problem may admit multiple solutions).
Since $\boldsymbol{\varepsilon}$ is independent of $\boldsymbol{\xi}$ by construction, the two channels carry no information about each other once $\mathbf{u}$ is given. A formal marginalization argument is provided in Appendix~\ref{app:cond-indep}.

Conditional independence given the clean wavefield does not, in general, transfer to the noisy state $\mathbf{u}_\tau$: residual uncertainty about $\mathbf{u}$ at diffusion time $\tau$ couples the two observation channels through the prior. We adopt the modeling choice of factorizing the noisy state likelihood as
\begin{equation}
    \label{eq:noisy-factor}
    p(\mathbf{y}, \mathbf{\uNO} \mid \mathbf{u}_\tau)
    \;\approx\;
    p(\mathbf{y} \mid \mathbf{u}_\tau)\, p(\mathbf{\uNO} \mid \mathbf{u}_\tau),
\end{equation}
Under the standard DPS delta-mass collapse $p(\mathbf{u} \mid \mathbf{u}_\tau) \approx \delta(\mathbf{u} - \bar{\mathbf{u}}_0)$, the factorization is exact, so the assumption made here is no stronger than DPS's. A formal derivation, together with the exact joint score from which~\eqref{eq:noisy-factor} departs, is given in Appendix~\ref{app:modeling-choice}.

Applying Bayes' rule together with the Markov structure of the forward diffusion, the posterior score at each diffusion time admits the decomposition 
\begin{equation}
    \label{eq:posterior-score-three-terms}
    \nabla_{\mathbf{u}_\tau}\!\log p_\tau(\mathbf{u}_\tau \mid
      \mathbf{y}, \mathbf{\uNO})
    \;=\;
    \underbrace{
      \nabla_{\mathbf{u}_\tau}\!\log p_\tau(\mathbf{u}_\tau)
    }_{\text{prior}}
    \;+\;
    \underbrace{
      \nabla_{\mathbf{u}_\tau}\!\log p(\mathbf{y} \mid \mathbf{u}_\tau)
    }_{\text{sensor}}
    \;+\;
    \underbrace{
      \nabla_{\mathbf{u}_\tau}\!\log p(\mathbf{\uNO} \mid \mathbf{u}_\tau)
    }_{\text{NO}}.
\end{equation}
where the equality is exact under~\eqref{eq:noisy-factor}.

The prior score is given by the pretrained denoiser through Tweedie's formula~\eqref{eq:tweedie}, while the sensor score coincides with the standard DPS expression ~\eqref{eq:dps-likelihood-score}. The remainder of this section is devoted to the NO term $\nabla_{\mathbf{u}_\tau}\!\log p(\mathbf{\uNO} \mid \mathbf{u}_\tau)$ for which we derive a closed-form spectral expression that accounts for the mode-dependent accuracy of the NO prediction.

\paragraph{Spectral observation model.}
Neural operators exhibit \emph{spectral bias}: low-frequency content is reproduced more accurately than high-frequency modes, which are systematically attenuated or not captured. 
A model of the surrogate error with uniform power spectral density (white noise) is therefore structurally inadequate, since the surrogate error is concentrated at high frequencies. We instead model the NO prediction in the Fourier domain. The model that follows uses only the second-order statistics of paired $(\mathbf{u},\mathbf{\uNO})$ data; we make no assumption on the distributional form of the clean signal or the surrogate residual, and no problem-specific physics enters the guidance term beyond what is contained in these statistics. Let $\mathcal{F}$ denote the unitary Fourier transform applied jointly over $(N_x, N_y, T)$ per channel, with modes indexed by $k = (k_x,k_y,k_t)$. At each mode and channel we decompose
\begin{equation}
\label{eq:per-mode-model}
    \mathcal{F}(\mathbf{\uNO})(k)
    \;=\;
    H(k)\,\mathcal{F}(\mathbf{u})(k)
    \;+\;
    \hat{\eta}_{\mathrm{NO}}(k),
\end{equation}
where $H(k) \in \mathbb{R}$ is a deterministic mode-dependent transfer function and $\hat{\eta}_{\mathrm{NO}}(k)$ is a zero-mean random residual with mode-dependent variance $\sigma^2_{\mathrm{NO}}(k) \in \mathbb{R}_{>0}$; both $H(k)$ and $\sigma^2_{\mathrm{NO}}(k)$ are parameters of the model, estimated from paired data as described below. $H(k)$ is specifically the LMMSE regression coefficient of $\mathcal{F}(\mathbf{\uNO})(k)$ on $\mathcal{F}(\mathbf{u})(k)$, i.e.\ the value minimizing $\mathbb{E}\bigl[|\mathcal{F}(\mathbf{\uNO})(k) - H(k)\,\mathcal{F}(\mathbf{u})(k)|^2\bigr]$. Concretely, $H(k)$ is computed from paired data as the cross-spectral ratio
\begin{equation}
\label{eq:H-estimator}
    H(k)
    \;=\;
    \mathrm{Re}\,
    \frac{
      N^{-1}\sum_{n=1}^{N}
      \mathcal{F}(\mathbf{\uNO}^{(n)})(k)\;
      \overline{\mathcal{F}(\mathbf{u}^{(n)})(k)}
    }{
      P_{\mathbf{u}}(k)
    },
\end{equation}
where the real-part reduction is justified empirically
(Appendix~\ref{app:spectral-empirical}).
By construction (see Appendix~\ref{app:wss} for further details), the residual  $\hat{\eta}_{\mathrm{NO}}(k)$ is then orthogonal to $\mathcal{F}(\mathbf{u})(k)$:
\begin{equation}
\label{eq:residual-orthogonality}
    \mathbb{E}\!\left[\hat{\eta}_{\mathrm{NO}}(k)\right] = 0,
    \qquad
    \mathbb{E}\!\left[\hat{\eta}_{\mathrm{NO}}(k)\,
        \overline{\mathcal{F}(\mathbf{u})(k)}\right] = 0,
    \qquad
    \mathbb{E}\!\left[|\hat{\eta}_{\mathrm{NO}}(k)|^2\right] 
    = \sigma^2_{\mathrm{NO}}(k).
\end{equation}
The transfer function $H$ thus absorbs all systematic, second-order linear structure of the NO error at mode $k$, so that the residual $\hat{\eta}_{\mathrm{NO}}$ carries no linear predictability from the clean signal. One expects $|H(k)| \approx 1$ at low frequencies and $|H(k)| \ll 1$ in the spectral-bias regime.
We assume the real-space residual field $\eta_{\mathrm{NO}}(x,y,t)$ is wide-sense stationary (WSS) in $(x,y,t)$: its mean is constant in $(x,y,t)$ and its autocovariance $\mathbb{E}[\eta_{\mathrm{NO}}(x,y,t)\,\eta_{\mathrm{NO}}(x+\Delta x, y+\Delta y, t+\Delta t)]$ depends only on the spatio-temporal lag $(\Delta x, \Delta y, \Delta t)$, not on the absolute position $(x,y,t)$. The expectations are taken over the training ensemble of geologies and source locations, and the approximation is supported by the shift-invariant structure of the FNO's spectral truncation error.
By the Wiener--Khinchin theorem, WSS implies that the residual covariance is diagonal in the Fourier basis with entries
$\sigma^2_{\mathrm{NO}}(k) \in \R_{>0}$.
This diagonality is a verifiable property of the surrogate: we confirm it for the MIFNO via an off-diagonal cross-spectral coherence diagnostic, and the same diagnostic serves as a prerequisite check for applying the method to any new surrogate (Appendix~\ref{app:wss-verification}).
The three mode-dependent quantities, $H(k)$, $\sigma^2_{\mathrm{NO}}(k)$, and the ensemble-averaged signal power spectrum
$P_{\mathbf{u}}(k) \coloneqq N^{-1}\sum_{n=1}^{N}|\mathcal{F}(\mathbf{u}^{(n)})(k)|^2$, are estimated from paired ground-truth/MIFNO data on a held-out calibration split and frozen as lookup tables at inference time.
Explicit estimator formulas and empirical validation of the spectral model are given in Appendix~\ref{app:spectral-validation}.

\paragraph{Spectrally shaped diffusion posterior via LMMSE.}

Evaluating the NO likelihood at the noisy state $\mathbf{u}_\tau$
requires the marginal $p(\mathbf{\uNO} \mid \mathbf{u}_\tau)$,
obtained by integrating out the unknown clean state:
$p(\mathbf{\uNO} \mid \mathbf{u}_\tau) = \int
p(\mathbf{\uNO} \mid \mathbf{u})\,
p(\mathbf{u} \mid \mathbf{u}_\tau)\,d\mathbf{u}$.
Standard DPS collapses $p(\mathbf{u} \mid \mathbf{u}_\tau)$
onto a delta mass at $\bar{\mathbf{u}}_0$, which assigns a
residual uncertainty of $\sigma^2_\tau$ per spatial degree of
freedom.
Since $\mathcal{F}$ is unitary, this isotropic spatial variance maps to a uniform variance $\sigma^2_\tau$ at every Fourier mode, a poor approximation for wavefields whose power spectrum spans orders of magnitude: it places the strongest NO guidance at high frequencies, precisely where the NO is least informative.

We correct this by replacing the isotropic approximation with the linear minimum mean squared error (LMMSE) estimator of the clean Fourier coefficient $X(k) \coloneqq \mathcal{F}(\mathbf{u})(k)$ from the noisy one $Y(k) \coloneqq \mathcal{F}(\mathbf{u}_\tau)(k) = X(k) + \hat{\eta}_\tau(k)$ with $\hat{\eta}_\tau(k) \sim \mathcal{CN}(0, \sigma^2_\tau)$. 
The derivation requires only the first two moments of $X$: zero mean (from z-score normalization of the training data) and variance $P_{\mathbf{u}}(k)$; no distributional assumption on $X$ is made.
The LMMSE estimate is $\hat{X}_L(k) := \alpha(k)\,Y(k)$, where
$\alpha(k)$ is the Wiener filter coefficient that optimally
shrinks the noisy observation toward zero (the prior mean),
and $\sigma_L^2(k)$ is the mean squared error of this estimate:
\begin{equation}
    \label{eq:alpha-and-SigmaL}
    \alpha(k)
    \;=\;
    \frac{P_{\mathbf{u}}(k)}{\sigma^2_\tau + P_{\mathbf{u}}(k)},
    \qquad
    \sigma_L^2(k)
    \;=\;
    \frac{\sigma^2_\tau\,P_{\mathbf{u}}(k)}
         {\sigma^2_\tau + P_{\mathbf{u}}(k)}
    \;=\;
    \sigma^2_\tau\,\alpha(k).
\end{equation}

The LMMSE error satisfies $\sigma_L^2(k) \leq \min\!\big(P_{\mathbf{u}}(k),\,\sigma^2_\tau\big)$: at low frequencies $\alpha(k) \to 1$ and the standard DPS posterior is recovered, while at high frequencies the posterior uncertainty is correctly bounded by the signal power rather than the diffusion noise level (full derivation in Appendix~\ref{app:lmmse}).

\paragraph{Isotropic vs.\ spectrally shaped guidance.}
Figure~\ref{fig:expected-guidance} shows the expected per-mode NO guidance magnitude, comparing the spectrally shaped score to the isotropic DPS guidance
across six noise levels.
Under the isotropic baseline ($H(k)=1$, $\sigma^2_{\mathrm{NO}}(k) = \sigma^2_{\mathrm{NO,iso}}$, $\alpha(k) = 1$), the
expected guidance magnitude reduces to the frequency-independent scalar $4/(\sigma^2_{\mathrm{NO,iso}} + \sigma^2_\tau)$, which assigns equal weight to
every mode at each noise level.
The spectrally shaped guidance tracks the isotropic level at low $\|k\|$, where the surrogate is accurate, and is progressively suppressed at high $\|k\|$, where
the surrogate is unreliable, the two differing by several orders of magnitude at high diffusion times. The upturn of the low-noise curves ($\sigma_\tau \leq 0.01$) at high $\|k\|$ is an edge effect at vanishing noise and does not affect the reverse trajectory (Appendix~\ref{app:spectral-empirical}).

\begin{figure}[t]
  \centering
  \includegraphics[width=0.65\textwidth]{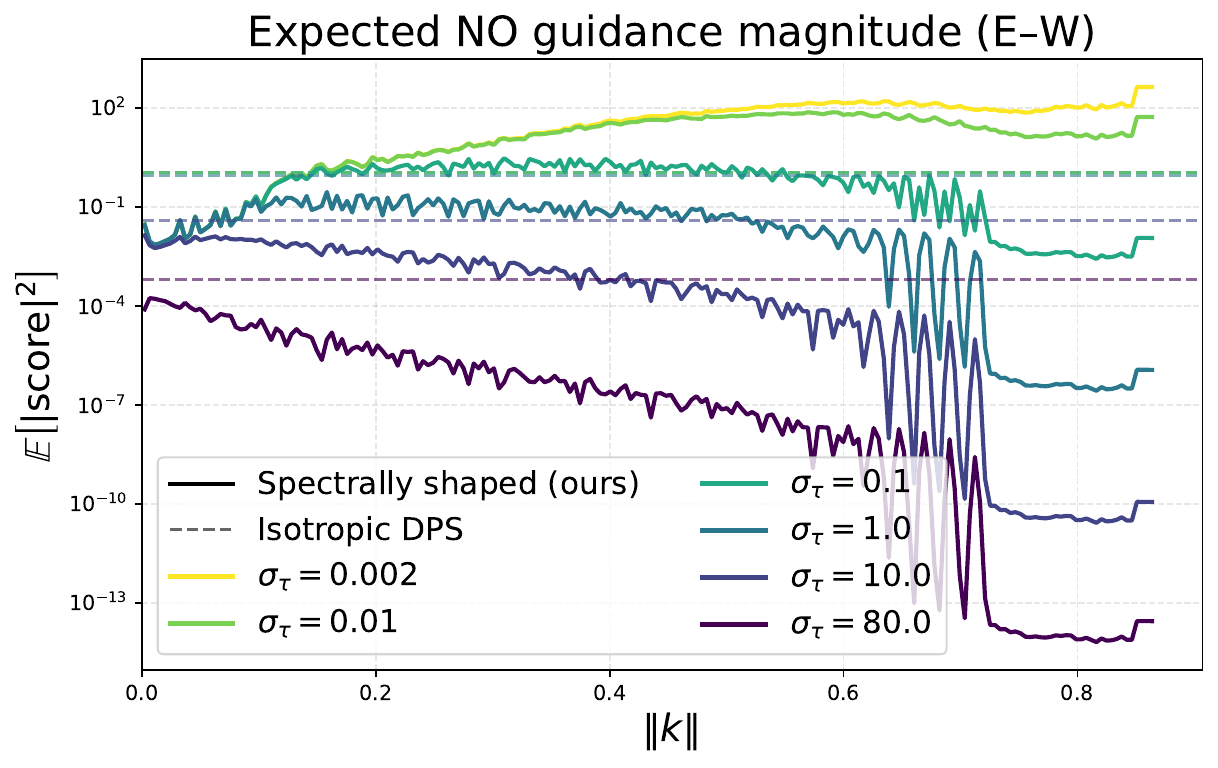}
    \caption{Expected NO guidance magnitude at six noise levels across the diffusion schedule (E--W component). Solid: spectrally shaped guidance, with calibrated $H(k)$, $\sigma^2_{\mathrm{NO}}(k)$, and LMMSE Wiener filter.
    Dashed: isotropic DPS guidance ($H\!=\!1$,    $\sigma^2_{\mathrm{NO}}\!=\!\sigma^2_{\mathrm{NO,iso}}$), frequency-independent at each noise level. At low $\|k\|$ both are comparable;
    at high $\|k\|$ the spectrally shaped guidance is suppressed by several orders of magnitude. The oscillations and subsequent flattening at high $\|k\|$ reflect the $5\,$Hz low pass filtering of the data, where the signal power reaches numerical floor and the per-mode ratio is no longer meaningful; this does not affect sampling, since the guidance there is non-negligible only at
    low $\sigma_\tau$, by which point the reconstruction is already essentially
    complete.}
  \label{fig:expected-guidance}
\end{figure}

\paragraph{Closed-form marginal likelihood for the NO term.}
Evaluating the NO score in~\eqref{eq:posterior-score-three-terms} requires the marginal $p(\mathbf{\uNO} \mid \mathbf{u}_\tau)$, obtained by integrating out the unknown clean state.
Recalling that $X(k)\!=\!\mathcal{F}(\mathbf{u})(k)$,
$Y(k)\!=\!\mathcal{F}(\mathbf{u}_\tau)(k)$ and introducing 
$Z(k)\!=\!\mathcal{F}(\mathbf{\uNO})(k)$,
we compute the conditional mean and variance of $Z(k)$ given $Y(k)$ by propagating
the LMMSE posterior through the per-mode spectral decomposition~\eqref{eq:per-mode-model}.
Since $Z(k) = H(k) X(k) + \hat{\eta}_{\mathrm{NO}}(k)$, with $\hat{\eta}_{\mathrm{NO}}(k)$ uncorrelated with $X(k)$ (by LMMSE definition of $H$) and independent of $\hat{\eta}_\tau$, standard second-order calculations give
\begin{align}
  \label{eq:marginal-mean-short}
  \mathbb{E}_L[Z(k) \mid Y(k)] &= H(k)\,\alpha(k)\,Y(k), \\[4pt]
  \label{eq:marginal-var-short}
  \lambda_\tau(k)
    \;\coloneqq\; \mathbb{E}\!\left[|Z(k) - \mathbb{E}_L[Z(k)\mid Y(k)]|^2\right]
    &= |H(k)|^2\,\sigma_L^2(k) \;+\; \sigma^2_{\mathrm{NO}}(k).
\end{align}
Both moments are \emph{exact}: they follow from independence and second-order statistics alone, with no distributional assumption on $X$. We approximate the full marginal by the complex normal distribution matching these moments:
\begin{equation}
  \label{eq:marginal-likelihood}
  p\!\left(\mathcal{F}(\mathbf{\uNO})(k) \mid \mathbf{u}_\tau\right)
  \;\approx\;
  \mathcal{CN}\!\left(
    H(k)\,\alpha(k)\,\mathcal{F}(\mathbf{u}_\tau)(k),\;\;
    \lambda_\tau(k)
  \right),
\end{equation}
with the joint distribution factorizing as a product of scalar complex Gaussians over all modes and channels.

\paragraph{NO likelihood score.}

Since the marginal mean in~\eqref{eq:marginal-likelihood} is linear in $\mathbf{u}_\tau$ in the Fourier domain, the likelihood score admits a closed-form expression requiring no backpropagation through the denoiser.
Defining the spectral residual
$\mathbf{r}(k) \coloneqq \mathcal{F}(\mathbf{\uNO})(k) -
H(k)\,\alpha(k)\,\mathcal{F}(\mathbf{u}_\tau)(k)$,
the score is
\begin{equation}
  \label{eq:no-score}
  \nabla_{\mathbf{u}_\tau}\!\log p(\mathbf{\uNO} \mid \mathbf{u}_\tau)
  \;=\;
  \mathcal{F}^{-1}\!\left(
    \tilde{\lambda}_\tau \odot \mathbf{r}
  \right),
\end{equation}
with the spectral weighting filter
\begin{equation}
  \label{eq:spectral-weight}
  \tilde{\lambda}_\tau(k)
  \;\coloneqq\;
  \frac{2\,H(k)\,\alpha(k)}{\lambda_\tau(k)}
  \;=\;
  \frac{
    2\,H(k)\,P_{\mathbf{u}}(k)
  }{
    \left(\sigma^2_\tau + P_{\mathbf{u}}(k)\right)
      \sigma^2_{\mathrm{NO}}(k)
    \;+\;
    \sigma^2_\tau\,|H(k)|^2\,P_{\mathbf{u}}(k)
  }.
\end{equation}
This is the LMMSE approximation of an exact, distribution-free score identity that we establish below (Prop.~\ref{prop:exact-score-main}). 
We note that Proposition~\ref{prop:exact-score-main} and the error analysis of Sec.~\ref{sec:score-analysis} are stated in terms of the per-mode Wirtinger derivative $\nabla_{Y^*}\log p(Z\mid Y) = H^*\alpha\,r/\lambda_\tau$, whereas the real-space gradient driving the sampler in~\eqref{eq:no-score} is $\nabla_{\mathbf{u}_\tau} = \mathcal{F}^{-1}\!\bigl(2\,H^*\alpha\,r/\lambda_\tau\bigr)$. The two differ by the global factor $2$ arising from $d|r|^2 = 2\,\mathrm{Re}[\bar{r}\,dr]$; this constant is immaterial to the relative-error bounds and is in practice absorbed into a constant.

The weight $\tilde{\lambda}_\tau(k)$ is large at low frequencies, where
the NO is faithful and the signal is strong, and vanishes at high
frequencies through both the Wiener suppression $\alpha(k) \to 0$ in
the numerator and the bounded posterior variance $\sigma_L^2(k) \leq P_{\mathbf{u}}(k)$ in the denominator.
Evaluation requires only two fast Fourier transforms (FFTs) and elementwise operations with
the precomputed quantities, making the cost negligible relative
to the denoiser forward pass.
The full gradient derivation is given in Appendix~\ref{app:no-score-derivation}. 

\paragraph{Combined posterior score.}

Substituting the prior score~\eqref{eq:tweedie}, the sensor
score~\eqref{eq:dps-likelihood-score}, and the NO
score~\eqref{eq:no-score}
into~\eqref{eq:posterior-score-three-terms}, the full posterior
score driving the reverse-time ODE is
\begin{equation}
  \label{eq:combined-score}
  \nabla_{\mathbf{u}_\tau}\!\log
    p_\tau(\mathbf{u}_\tau \mid \mathbf{y}, \mathbf{\uNO})
  \;=\;
  \underbrace{
    \frac{\bar{\mathbf{u}}_0 - \mathbf{u}_\tau}{\sigma_\tau^2}
  }_{\text{(i) prior}}
  \;+\;
  \underbrace{
    \mathcal{F}^{-1}\!\left(
      \tilde{\lambda}_\tau \odot \mathbf{r}
    \right)
  }_{\text{(ii) NO guidance}}
  \;-\;
  \underbrace{
    \frac{1}{\sigma_y^2}\,
    \left(\frac{\partial \bar{\mathbf{u}}_0}
               {\partial \mathbf{u}_\tau}\right)^{\!\top}
    \mathcal{M}_{\mathcal{S}}^\dagger\!\left(
      \mathcal{M}_{\mathcal{S}}(\bar{\mathbf{u}}_0)
      - \mathbf{y}
    \right)
  }_{\text{(iii) sensor guidance}}.
\end{equation}
The two guidance terms play complementary roles:
the NO term operates in the Fourier domain with principled spectral
weighting, supplying global structural information with
$\tau$-dependent annealing; the sensor term operates in the spatial domain through the denoiser Jacobian, enforcing sharp pointwise consistency at the observed locations.
At each reverse-diffusion step, the computational cost comprises
one denoiser forward pass (shared by terms~(i) and~(iii)),
one vector-Jacobian product through the denoiser (term~(iii)),
and two FFTs with elementwise operations (term~(ii), negligible).
The full sampling procedure is detailed in Algorithm~\ref{alg:dps-no} (Appendix~\ref{app:algorithm}).

\subsection{Exact score identity and approximation error}
\label{sec:score-analysis}
In practice, the spectral NO score in ~\eqref{eq:no-score} replaces the true marginal likelihood with a moment-matched Gaussian.
To characterize the error introduced by this approximation, we first establish an exact expression for the NO likelihood score that holds without any distributional assumption.
\begin{prop}[Exact NO likelihood score]
\label{prop:exact-score-main}
At each Fourier mode $k$, let $X = \mathcal{F}(\mathbf{u})(k)$, $Y = \mathcal{F}(\mathbf{u}_\tau)(k)$, and $Z = \mathcal{F}(\mathbf{\uNO})(k)$ under the per-mode models $Y = X + \hat{\eta}_\tau$ and $Z = HX + \hat{\eta}_{\mathrm{NO}}$, with $\hat{\eta}_{\mathrm{NO}}$ zero-mean with variance $\sigma^2_{\mathrm{NO}}$ but otherwise unrestricted. 
Then: 
\begin{equation}
  \label{eq:exact-score-main}
  \nabla_{Y^*}\log p(Z \mid Y)
  \;=\;
  \frac{\mathbb{E}[X \mid Y, Z]
    \;-\; \mathbb{E}[X \mid Y]}{\sigma^2_\tau}.
\end{equation}
\end{prop}
The proof proceeds by differentiating the marginal $p(Z \mid Y) = \int p(Z \mid X)\,p(X \mid Y)\,dX$ with respect to $Y^*$, applying Bayes' rule to identify $p(X \mid Y, Z)$ in the resulting integral, and recognizing the Tweedie score of the diffusion channel; the full derivation is given in Appendix~\ref{app:exact-score}.

The identity has a clear interpretation: the NO likelihood score measures how much the optimal estimate of the clean Fourier coefficient $X$ shifts when we additionally condition on the surrogate observation $Z$, normalized by the diffusion noise variance.
Where $Z$ is uninformative given $Y$, the shift vanishes and the score is zero; where $Z$ provides strong additional information, the shift is large and the score steers the reverse diffusion accordingly.

\paragraph{Score error under the Gaussian approximation.}
Our score formulation ~\eqref{eq:no-score} replaces both conditional expectations in~\eqref{eq:exact-score-main} with their LMMSE counterparts: $\mathbb{E}[X \mid Y] \approx \alpha Y$ and $\mathbb{E}[X \mid Y, Z] \approx \alpha Y + (H^*\sigma_L^2 / \lambda_\tau)\,r$.
The resulting score decomposes as
\begin{equation}
  \label{eq:score-error-main}
  \epsilon
  \;=\;
  \frac{1}{\sigma^2_\tau}\Bigl[
  \underbrace{
    \bigl(\mathbb{E}[X \mid Y, Z]
      - \mathbb{E}_L[X \mid Y, Z]\bigr)
  }_{\delta_{\mathrm{post}}}
  \;-\;
  \underbrace{
    \bigl(\mathbb{E}[X \mid Y] - \alpha Y\bigr)
  }_{\delta_{\mathrm{prior}}}
  \Bigr],
\end{equation}
where both $\delta_{\mathrm{prior}}$ and $\delta_{\mathrm{post}}$ are MMSE-LMMSE gaps: the difference between the optimal nonlinear estimator and the optimal linear estimator of $X$.
If $X$ were Gaussian \emph{and} $\hat{\eta}_{\mathrm{NO}}$ were Gaussian, both gaps would vanish identically and the approximation would be exact.
The score error is therefore driven entirely by the non-Gaussianity of the clean Fourier coefficients and the NO residual.
By the Pythagorean property of MMSE estimation, the mean-squared gaps are bounded by the corresponding LMMSE errors \begin{equation} \mathbb{E}[|\delta_{\mathrm{prior}}|^2] \leq
\sigma_L^2 \quad; \quad \mathbb{E}[|\delta_{\mathrm{post}}|^2] \leq \sigma_{L,YZ}^2
\label{eq:bounds}
\end{equation} both computable from the calibrated second-order statistics alone (Appendix~\ref{app:score-error}).
These bounds are \emph{distribution-agnostic}: they require no assumption on the form of $p(X)$ or $p(\hat{\eta}_{\mathrm{NO}})$ beyond the calibrated moments $H(k)$, $\sigma^2_{\mathrm{NO}}(k)$, and $P_{\mathbf{u}}(k)$.

\section{Results}
\label{sec:results}

We evaluate the spectral guidance framework on three-dimensional elastic wavefield prediction from the HEMEW\textsuperscript{S}-3D database, providing ablation studies of the components in our method.

\subsection{Experimental setup}
\label{sec:setup}

We evaluate all methods on a held-out test set of $N=1{,}000$ SEM
simulations drawn from the HEMEW\textsuperscript{S}-3D database
(Sec.~\ref{sec:sem-data}), with geologies and source configurations
unseen during training. The HEMEW\textsuperscript{S}-3D database is partitioned into three disjoint splits: $27{,}000$ samples for training the diffusion prior, $2{,}000$ for the spectral calibration and hyperparameter selection, and $1{,}000$ samples for all reported test metrics.
For each test sample, sparse observations are generated by selecting
a random subset of the $32\times32$ surface grid uniformly at random;
the same spatial mask is applied to all three velocity components and
retained across all time steps.
We report results at sensor densities $\rho=5\%$
($|\mathcal{S}|=51$) and $\rho=2\%$ ($|\mathcal{S}|=20$).
Posterior samples are generated by solving the probability-flow
ODE~\citep{karras2022elucidating} with $64$ steps using the explicit
Euler integrator.

Five configurations are compared, summarized in
Table~\ref{tab:main-results}:
\begin{itemize}
    \item \emph{MIFNO} is the frozen deterministic surrogate (Sec.~\ref{sec:mifno-surrogate}), using no diffusion and no sensor data.
    \item \emph{DPS} applies standard diffusion posterior sampling~\eqref{eq:dps-likelihood-score} conditioned on sparse measurements alone.
    \item \emph{DPS\,+\,NO\,(iso)} adds neural-operator guidance with isotropic treatment ($H(k)=1$, $\sigma^2_{\mathrm{NO}}(k) = \sigma^2_{\mathrm{NO,iso}}$), routing both terms through the denoiser Jacobian. 
    \item \emph{\FreqNODPS\,($\alpha\!=\!1$)} uses the calibrated spectral observation model ($H(k)$, $\sigma^2_{\mathrm{NO}}(k)$) but
    disables the LMMSE Wiener filter ($\alpha(k)=1$).
    \item \emph{\FreqNODPS} is the full spectrally shaped
    method~\eqref{eq:combined-score}, with calibrated $H(k)$,
    $\sigma^2_{\mathrm{NO}}(k)$, and LMMSE Wiener filter
    $\alpha(k)$.
\end{itemize}
All diffusion-based methods share the same unconditional prior and
sampling configuration.

Reconstruction quality is assessed through different metrics.
Pointwise accuracy is measured by the relative mean absolute error (rMAE) and relative root mean squared error (rRMSE), computed per sensor location over the temporal axis and averaged over all locations, components, and samples.
Spectral fidelity is measured by the banded relative FFT bias (rFFT) in three frequency ranges, low ($0$--$1$\,Hz), mid ($1$--$2$\,Hz), and high ($2$--$5$\,Hz), where negative values indicate systematic spectral underestimation and zero indicates unbiased reproduction.
We additionally report the significant-duration error SD\textsubscript{5--95}, which captures mismatches in the temporal energy distribution.
Full metric definitions are given in Appendix~\ref{app:experimental}.

\subsection{Quantitative comparison}
\label{sec:quantitative}

\begin{table}[t]
\centering
\caption{Quantitative comparison across sensor densities and
ablations.
MIFNO is sensor-independent; all other methods use the same
unconditional diffusion prior.
DPS\,+\,NO\,(iso) treats the surrogate isotropically;
\FreqNODPS\,($\alpha\!=\!1$) uses the calibrated spectral model
but disables the LMMSE Wiener filter;
\FreqNODPS\ is the full method.
Mean $\pm$ std over $1{,}000$ test samples.}
\label{tab:main-results}
\resizebox{\textwidth}{!}{
\begin{tabular}{@{}llcccccc@{}}
\toprule
& & rMAE\,$\downarrow$ & rRMSE\,$\downarrow$
& rFFT\textsubscript{low}\,$\to\!0$
& rFFT\textsubscript{mid}\,$\to\!0$
& rFFT\textsubscript{high}\,$\to\!0$
& SD\textsubscript{5--95}\,$\downarrow$ \\
\midrule
\multicolumn{2}{@{}l}{\textit{Sensor-independent}} \\[2pt]
& MIFNO
  & $0.133{\scriptstyle\,\pm\,}0.052$
  & $0.224{\scriptstyle\,\pm\,}0.085$
  & $-0.075{\scriptstyle\,\pm\,}0.173$
  & $-0.137{\scriptstyle\,\pm\,}0.180$
  & $-0.239{\scriptstyle\,\pm\,}0.207$
  & $0.581{\scriptstyle\,\pm\,}0.356$ \\
\midrule
\multicolumn{2}{@{}l}{\textit{$\rho=5\%$
  \;($|\mathcal{S}|=51$)}} \\[2pt]
& DPS
  & $0.113{\scriptstyle\,\pm\,}0.050$
  & $0.217{\scriptstyle\,\pm\,}0.082$
  & $-0.088{\scriptstyle\,\pm\,}0.093$
  & $-0.163{\scriptstyle\,\pm\,}0.122$
  & $-0.232{\scriptstyle\,\pm\,}0.124$
  & $0.148{\scriptstyle\,\pm\,}0.160$ \\[3pt]
& DPS\,+\,NO\, (iso)
  & $\mathbf{0.099}{\scriptstyle\,\pm\,}0.043$
  & $\mathbf{0.180}{\scriptstyle\,\pm\,}0.071$
  & $-0.060{\scriptstyle\,\pm\,}0.102$
  & $-0.117{\scriptstyle\,\pm\,}0.125$
  & $-0.210{\scriptstyle\,\pm\,}0.156$
  & $0.225{\scriptstyle\,\pm\,}0.165$ \\[3pt]
& \FreqNODPS\ ($\alpha\!=\!1$)
  & $0.110{\scriptstyle\,\pm\,}0.046$
  & $0.223{\scriptstyle\,\pm\,}0.085$
  & $\mathbf{-0.001}{\scriptstyle\,\pm\,}0.052$
  & $-0.048{\scriptstyle\,\pm\,}0.068$
  & $-0.066{\scriptstyle\,\pm\,}0.111$
  & $\mathbf{0.097}{\scriptstyle\,\pm\,}0.076$ \\[3pt]
& \FreqNODPS
  & $0.100{\scriptstyle\,\pm\,}0.045$
  & $0.200{\scriptstyle\,\pm\,}0.084$
  & $0.009{\scriptstyle\,\pm\,}0.036$
  & $\mathbf{-0.015}{\scriptstyle\,\pm\,}0.054$
  & $\mathbf{0.002}{\scriptstyle\,\pm\,}0.097$
  & $0.118{\scriptstyle\,\pm\,}0.082$ \\
\midrule
\multicolumn{2}{@{}l}{\textit{$\rho=2\%$
  \;($|\mathcal{S}|=20$)}} \\[2pt]
& DPS
  & $0.168{\scriptstyle\,\pm\,}0.053$
  & $0.280{\scriptstyle\,\pm\,}0.070$
  & $-0.418{\scriptstyle\,\pm\,}0.147$
  & $-0.506{\scriptstyle\,\pm\,}0.134$
  & $-0.557{\scriptstyle\,\pm\,}0.121$
  & $0.617{\scriptstyle\,\pm\,}0.355$ \\[3pt]
& DPS\,+\,NO\, (iso)
  & $\mathbf{0.119}{\scriptstyle\,\pm\,}0.047$
  & $\mathbf{0.203}{\scriptstyle\,\pm\,}0.075$
  & $-0.090{\scriptstyle\,\pm\,}0.141$
  & $-0.153{\scriptstyle\,\pm\,}0.160$
  & $-0.253{\scriptstyle\,\pm\,}0.186$
  & $0.350{\scriptstyle\,\pm\,}0.276$ \\[3pt]
& \FreqNODPS\ ($\alpha\!=\!1$)
  & $0.132{\scriptstyle\,\pm\,}0.053$
  & $0.232{\scriptstyle\,\pm\,}0.083$
  & $-0.118{\scriptstyle\,\pm\,}0.106$
  & $-0.192{\scriptstyle\,\pm\,}0.126$
  & $-0.250{\scriptstyle\,\pm\,}0.156$
  & $\mathbf{0.150}{\scriptstyle\,\pm\,}0.127$ \\[3pt]
& \FreqNODPS
  & $0.125{\scriptstyle\,\pm\,}0.050$
  & $0.235{\scriptstyle\,\pm\,}0.083$
  & $\mathbf{0.011}{\scriptstyle\,\pm\,}0.087$
  & $\mathbf{-0.006}{\scriptstyle\,\pm\,}0.120$
  & $\mathbf{0.047}{\scriptstyle\,\pm\,}0.191$
  & $0.286{\scriptstyle\,\pm\,}0.157$ \\
\bottomrule
\end{tabular}
}
\end{table}

\paragraph{Sensor-only DPS fails spectrally.}
At $\rho=5\%$, DPS provides essentially no spectral improvement over the deterministic surrogate: rFFT\textsubscript{high} is $-0.232$ compared to MIFNO's $-0.239$.
Despite access to ground-truth observations at $51$ locations, the likelihood gradient constrains too few spatial degrees of freedom to steer the diffusion prior toward the correct high-frequency content.
The failure is far more severe at  $\rho=2\%$, where DPS degrades \emph{below} the surrogate on all metrics (rMAE of $0.168$ vs. $0.133$; rFFT\textsubscript{high} of $-0.557$), losing over half the high-frequency spectral content.
With only $20$ sensors constraining $1{,}024$ spatial locations, the sparse likelihood becomes too weak to guide the reverse diffusion.

\paragraph{Isotropic NO guidance reimports spectral bias.}
DPS\,+\,NO\,(iso), which treats the surrogate prediction as an isotropic Gaussian observation, demonstrates that the surrogate's structural information is genuinely useful: at $\rho=5\%$ it substantially improves pointwise accuracy over sensor-only DPS (rMAE of $0.099$ vs.\ $0.113$; rRMSE of $0.180$ vs.\ $0.217$), and at $\rho=2\%$ it recovers from the collapse of sensor-only DPS (rMAE of $0.119$ vs.\ $0.168$).
However, uniform spectral weighting carries the surrogate's bias into the posterior nearly intact: rFFT\textsubscript{high} is
$-0.210$ at $\rho=5\%$ and $-0.253$ at $\rho=2\%$, barely improved from MIFNO's $-0.239$.
This confirms that without spectrally shaped calibration, NO guidance \emph{does} reimpose spectral bias at both sensor densities.

\paragraph{Spectral shaping resolves the trade-off.}
\FreqNODPS\ achieves near-zero spectral bias at both sensor densities: rFFT\textsubscript{high} of $+0.002$ at $\rho=5\%$ and $+0.047$ at $\rho=2\%$. At $\rho=5\%$ this is two orders of magnitude smaller in absolute value than DPS\,+\,NO\,(iso) ($+0.002$ vs.\ $-0.210$); at $\rho=2\%$ it is roughly $5\times$ smaller ($+0.047$ vs.\ $-0.253$).
At $\rho=5\%$, this spectral correction comes at no pointwise cost: rMAE matches DPS\,+\,NO\,(iso) within one standard error of the mean ($0.100$ vs.\ $0.099$; standard error $\approx 0.0014$ at $N=1{,}000$), so the two methods are statistically indistinguishable on pointwise accuracy while differing by two orders of magnitude on spectral fidelity.
At $\rho=2\%$, the spectral correction incurs a modest pointwise cost (rMAE $0.125$ vs.\ $0.119$, $\approx\!5\%$ relative; statistically significant at $N=1{,}000$), reflecting the stronger regularization required when sensor anchoring is weakest.
The spectrally calibrated guidance channels the surrogate's information at frequencies where it is reliable ($|H(k)| \approx 1$, $\gamma(k) \ll 1$) and vanishes where it is not, so that the diffusion prior is free to generate high-frequency detail uncorrupted by the surrogate's bias.
A sensitivity analysis over two orders of magnitude of $\lambda_{\mathrm{NO}}$ (Appendix~\ref{app:sensitivity}) confirms that the calibrated value sits at the natural zero-crossing of the spectral correction, with pointwise accuracy essentially flat across a $20\times$ range of $\lambda_{\mathrm{NO}}$ around it.

\paragraph{LMMSE provides the final spectral correction.}
The \FreqNODPS\,($\alpha\!=\!1$) ablation retains the calibrated spectral model, $H(k)$ and $\sigma^2_{\mathrm{NO}}(k)$, but replaces the LMMSE Wiener filter with the isotropic posterior approximation ($\alpha(k)=1$). At $\rho=5\%$, this alone brings rFFT\textsubscript{high} from
$-0.210$ (DPS\,+\,NO\,(iso)) to $-0.066$: the frequency-dependent observation model accounts for most of the spectral correction.
The LMMSE layer pushes rFFT\textsubscript{high} to $+0.002$, eliminating the residual bias.
At $\rho=2\%$, however, the $\alpha\!=\!1$ ablation shows rFFT\textsubscript{high} of $-0.250$, nearly identical to DPS\,+\,NO\,(iso), indicating that the spectral observation model alone is insufficient at extreme sparsity and the LMMSE correction becomes essential to achieve the full spectral correction (rFFT\textsubscript{high} of $+0.047$).

\subsection{Posterior stability and calibration}
\label{sec:posterior}

The main reconstruction results in Section~\ref{sec:quantitative}
are generated with the probability-flow ODE, which is the efficient
choice for point-estimate evaluation. For calibration analysis,
however, an ODE sampler is inappropriate: the integrator is deterministic given the initial noise and produces an artificially
narrow empirical distribution that underestimates posterior uncertainty. We therefore switch to an SDE sampler for the experiments in this section, which preserves stochasticity throughout the reverse trajectory. We generate $M=20$ posterior samples for $100$ test cases at $\rho=5\%$, each sharing the same sensor mask and MIFNO prediction but initialized with independent noise draws. Table~\ref{tab:posterior}
compares \FreqNODPS\ with sensor-only DPS and DPS\,+\,NO\,(iso) under SDE sampling.

\begin{table}[t]
    \centering
    \caption{Posterior stability and calibration at $\rho=5\%$
    under SDE sampling ($M=20$ realizations, $100$ test samples).
    \emph{Coverage}: fraction of ground-truth values within the
    pointwise $\pm 2\hat{\sigma}$ interval across realizations
    (nominal $0.954$); bold marks the value closest to nominal.
    \emph{CI width}: mean pointwise interval width $4\hat{\sigma}$
    normalized by the per-sample trace amplitude.
    \emph{Posterior std}: mean across-realization standard
    deviation in z-score normalized space.
    \emph{rMAE}: mean $\pm$ std across the $M$ realizations.
    CI width, posterior std, and rMAE are reported without a
    best-direction highlight, as smaller interval widths are only
    desirable when coverage is maintained.}
    \label{tab:posterior}
    \small
    \begin{tabular}{@{}lcccc@{}}
    \toprule
    & Coverage & CI width & Posterior std & rMAE \\
    \midrule
    DPS
      & $0.842$
      & $0.022$
      & $1.6 \times 10^{-3}$
      & $0.134 {\scriptstyle\,\pm\,} 0.009$ \\[3pt]
    DPS\,+\,NO\,(iso)
      & $0.769$
      & $0.015$
      & $1.1 \times 10^{-3}$
      & $0.110 {\scriptstyle\,\pm\,} 0.004$ \\[3pt]
    \FreqNODPS
      & $\mathbf{0.854}$
      & $0.018$
      & $1.4 \times 10^{-3}$
      & $0.112 {\scriptstyle\,\pm\,} 0.001$ \\
    \bottomrule
    \end{tabular}
\end{table}

\FreqNODPS\ exhibits the lowest inter-realization variability of the three methods: the rMAE standard deviation is $0.001$ on a mean of $0.112$, a coefficient of variation of roughly $1\%$. Sensor-only DPS is more than five times more variable (rMAE std of $0.009$ on a mean of $0.134$; CoV $\approx 7.5\%$), and DPS\,+\,NO\,(iso) sits in between (std of $0.004$, CoV $\approx 4\%$). The spectrally calibrated guidance therefore not only improves the spectral profile of individual reconstructions (Table~\ref{tab:main-results}) but also
stabilizes the reconstruction trajectory across stochastic initializations. The gap between methods remains the dominant source of variation: under SDE
sampling (Table~\ref{tab:posterior}), the rMAE difference between sensor-only DPS and \FreqNODPS\ ($0.134$ vs.\ $0.112$, a gap of $0.022$) is more than $20\times$ the inter-realization spread of \FreqNODPS\ ($0.001$).

\paragraph{Posterior concentration.}
\FreqNODPS\ achieves the highest empirical coverage of the three methods ($85.4\%$ vs.\ $84.2\%$ for sensor-only DPS and $76.9\%$ for DPS\,+\,NO\,(iso)), while simultaneously producing narrower credible intervals than sensor-only DPS (CI width of $0.018$ vs.\
$0.022$). The combination is informative because coverage can be trivially inflated by widening the intervals. \FreqNODPS\ instead achieves higher coverage with a tighter posterior, so the spectral guidance is producing a posterior that is both better calibrated and more concentrated than the alternatives.

DPS\,+\,NO\,(iso) shows the opposite pattern: it has the narrowest intervals (CI width $0.015$) and the lowest posterior standard deviation, but also the worst coverage. The isotropic surrogate guidance pulls every spatiotemporal point toward the MIFNO prediction, contracting the empirical distribution around a biased mean and leaving more ground-truth values outside the $\pm 2\hat{\sigma}$ band. This mirrors the spectral-bias finding from Section~\ref{sec:quantitative}: uniform spectral weighting carries the surrogate's deficiencies into the posterior, here manifesting as miscalibration rather than as attenuated frequencies.

\paragraph{Residual gap to nominal coverage.}
All three methods fall below the nominal $95\%$ coverage of a correctly calibrated $\pm 2\hat{\sigma}$ band. The residual gap of roughly $10$ percentage points for \FreqNODPS\ is consistent across nominal levels: the $1\sigma$ coverage is $58.2\%$ versus nominal $68.3\%$, a gap of $10.1$ points nearly identical to the $2\sigma$ gap. This consistency suggests a uniform underdispersion of the credible intervals rather than a tail-specific miscalibration, compatible with the moment-matching Gaussian approximation in the NO likelihood (Section~\ref{sec:no-guided-dps}) being mildly too light-tailed at all scales. Closing this gap would require either a refined moment model that captures the heavy-tailed behaviour of the underlying distribution, larger ensembles for tail estimation, or both; we leave a systematic study of calibrated credible intervals to future work.

\subsection{Wavefield reconstruction quality}
\label{sec:qualitative}
Figure~\ref{fig:snapshots} shows velocity maps for a
representative test sample at $\rho=5\%$.
The MIFNO prediction reproduces the large-scale wavefront geometry but exhibits the characteristic oversmoothing induced by spectral bias, with attenuated amplitudes and blurred wavefronts.
DPS-NO (iso) alone recovers the overall structure of the field but lacks spectral consistency.
The \FreqNODPS\ combines the global structural prior from the calibrated neural operator with the local truth anchoring from the sensors, producing reconstructions that are visually closest to the ground truth across the entire spatial domain.

\begin{figure}[t]
  \centering
  \includegraphics[width=\textwidth]{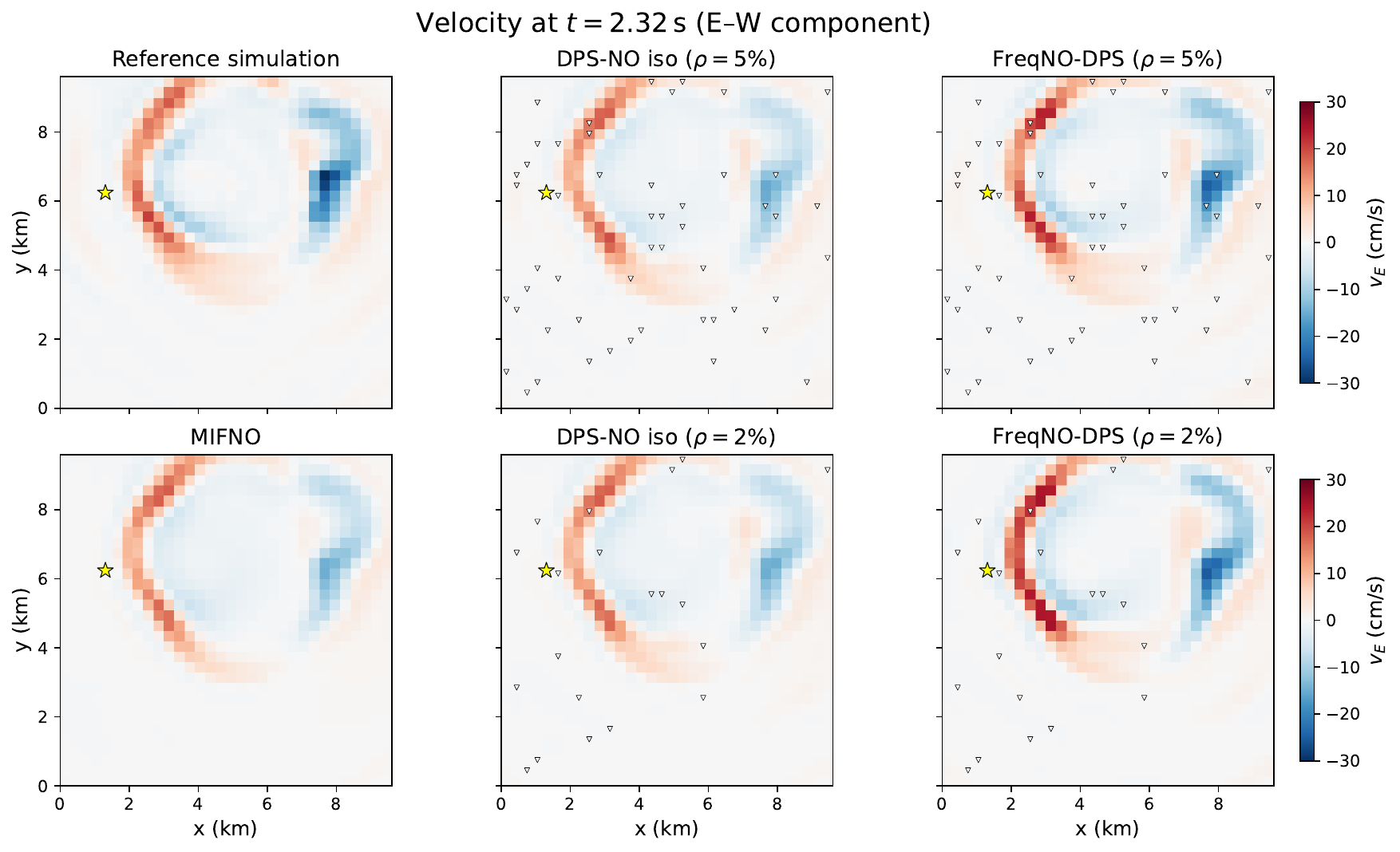}
    \caption{East--west velocity at $t = 2.32$\,s for    a representative test sample.
      Left column: sensor-independent references (top: SEM simulation, bottom: MIFNO surrogate).
      Middle column: DPS\,+\,NO\,(iso) at $\rho=5\%$ (top) and $\rho=2\%$ (bottom).
      Right column: \FreqNODPS\ at $\rho=5\%$ (top) and $\rho=2\%$ (bottom).
      Yellow star: source location; white triangles: sensor positions.}
  \label{fig:snapshots}
\end{figure}

Figure~\ref{fig:traces} compares the vertical velocity time histories at the sensor recording the highest peak ground velocity.
The MIFNO trace captures the dominant arrival but underestimates peak amplitudes and lacks small-amplitude oscillations.
The \FreqNODPS\ tracks the ground-truth waveform more closely in both phase and amplitude.

\begin{figure}[t]
  \centering
  \includegraphics[width=\textwidth]{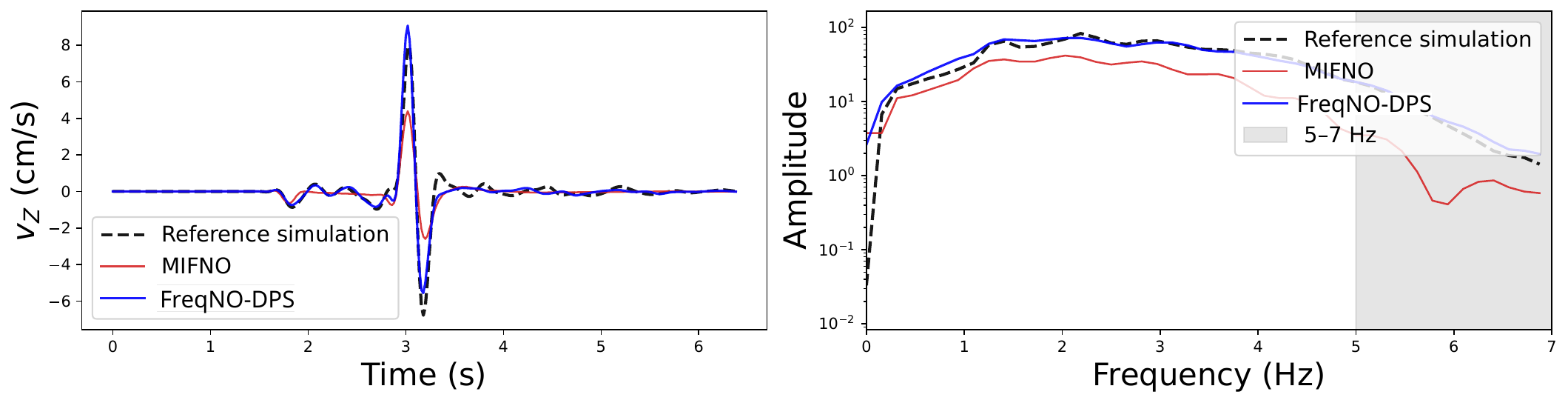}
  \caption{Vertical component velocity time histories at a sensor location ($\rho=5\%$) and the corresponding spectrum.
  Black dashed: reference simulation; red: MIFNO; blue \FreqNODPS.}
  \label{fig:traces}
\end{figure}

Figure~\ref{fig:spectra} shows the ensemble-averaged frequency
spectrum at $\rho=5\%$, computed along the temporal axis at each
spatial location and averaged over all locations, components, and
$1{,}000$ test samples.
MIFNO and DPS\,+\,NO\,(iso) both attenuate spectral content above
${\sim}1$\,Hz, with their curves peeling away from the reference
at progressively higher frequencies.
\FreqNODPS\,($\alpha\!=\!1$) partially recovers high-frequency
content but still underestimates the overall spectrum.
\FreqNODPS\ tracks the reference spectrum across the full frequency
range, consistent with the near-zero rFFT values in
Table~\ref{tab:main-results}.

\begin{figure}[t]
  \centering
  \includegraphics[width=0.6\textwidth]{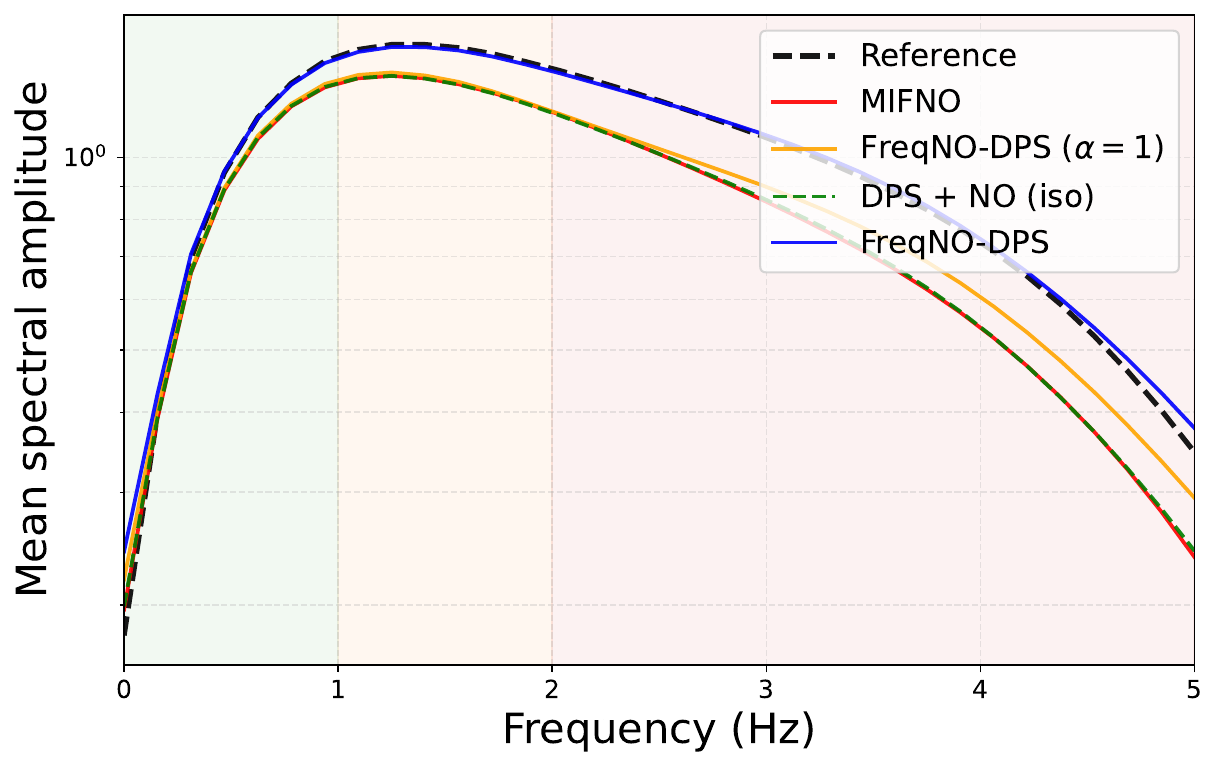}
  \caption{Ensemble-averaged frequency spectrum at $\rho=5\%$.
  For each method, $|\mathcal{F}(\mathbf{u})(f)|$ is computed per
  spatial location along the temporal axis and averaged over all
  locations, velocity components, and $1{,}000$ test samples.
  Black dashed: reference simulation; red: MIFNO; green:
  DPS\,+\,NO\,(iso); orange: \FreqNODPS\,($\alpha\!=\!1$); blue:
  \FreqNODPS.
  Shaded bands correspond to the low, mid, and high rFFT ranges
  in Table~\ref{tab:main-results}.}
  \label{fig:spectra}
\end{figure}

\section{Discussion}
\label{sec:discussion}

The experiments of Section~\ref{sec:results} show that spectral shaping is what separates effective neural operator guidance from a faithful reimport of the
surrogate's bias. We now examine why the moment-matched approximation underlying that guidance is well-behaved, and delineate the conditions under which the
approach applies. Section~\ref{sec:regime-analysis} characterizes the approximation error of the spectral NO score across the wavenumber--diffusion-time plane, building on the exact score identity and distribution-free bounds established in Section~\ref{sec:score-analysis}; Section~\ref{sec:limitations} discusses the scope of the spectral observation model, calibration under distribution shift, and sampling cost.

\subsection{Regime analysis}
\label{sec:regime-analysis}
We write $s^{\mathrm{approx}}$ for the moment-matched NO score~\eqref{eq:no-score} and $\epsilon \coloneqq s^{\mathrm{exact}} - s^{\mathrm{approx}}$~\eqref{eq:score-error-main}
for its error; the regime statements below concern their expected squared magnitudes $\mathbb{E}[|s^{\mathrm{approx}}|^2]$ and $\mathbb{E}[|\epsilon|^2]$, tabulated in Table~\ref{tab:regime-summary}.
These bounds depend on the diffusion noise level $\sigma_\tau$ and the mode-dependent NO accuracy through two dimensionless parameters: the inverse
per-mode diffusion SNR $\nu(k,\tau) \coloneqq \sigma^2_\tau / P_{\mathbf{u}}(k)$ and the relative NO error $\gamma(k) \coloneqq \sigma^2_{\mathrm{NO}}(k) / (|H(k)|^2 P_{\mathbf{u}}(k))$.
A third quantity, the smoothing ratio $\zeta(k,\tau) \coloneqq |H|^2 \sigma_L^2 / \sigma^2_{\mathrm{NO}}
= \nu / [\gamma(1+\nu)]$, controls the relative contribution of the marginalization uncertainty to the innovation variance.
Note that $\nu$ and $\zeta$ depend on both the wavenumber $k$ and the diffusion time $\tau$, while $\gamma$ depends on $k$ alone; we suppress these arguments hereafter for readability.
Together, $\nu$ and $\gamma$ partition the wavenumber--diffusion-noise plane into four regimes (Fig.~\ref{fig:regime-diagram}, Table~\ref{tab:regime-summary}), in each of which the score error is controlled by the calibrated second-order statistics alone. The detailed per-regime derivations, including all asymptotic expansions, are provided in Appendix~\ref{app:regime-details}.

\paragraph{Regime~I ($\nu \gg 1$): noise-dominated.} This arises at early diffusion times or for high-frequency modes where the spectral amplitude is smaller than the diffusion noise, i.e., $P_{\mathbf{u}}(k) \ll \sigma^2_\tau$.
Both $\mathbb{E}[|s^{\mathrm{approx}}|^2]$ and $\mathbb{E}[|\epsilon|^2]$ are $O(\nu^{-2})$ and vanish: the reverse diffusion is driven entirely by the
unconditional prior. 

\paragraph{Regime~II ($\nu \ll 1$, $\zeta \ll 1$): active guidance, spectral shape preserved.} This is the operating regime where the NO guidance is the strongest. The absolute score-error bound is finite, but the relative bound $O(\gamma/\nu)$ is not controlled by the analysis alone. The triangle inequality applied to~\eqref{eq:score-error-main} discards the partial cancellation between $\delta_{\mathrm{post}}$ and $\delta_{\mathrm{prior}}$
that arises when the surrogate observation adds little information about the clean Fourier coefficient beyond what the noisy diffusion state already provides. The spectral weighting filter $\tilde{\lambda}_\tau(k)$ that sets how strongly the guidance acts at each frequency is still distribution-agnostic in shape: its
variation across wavenumbers is determined by the calibrated quantities $H(k)$, $\sigma^2_{\mathrm{NO}}(k)$, and $P_{\mathbf{u}}(k)$, and non-Gaussianity can only rescale how much guidance each mode receives, not shift guidance from one frequency to another. Any such rescaling is absorbed into the empirical hyperparameter $\lambda_{\mathrm{NO}}$.
This prediction is verified empirically in Appendix~\ref{app:sensitivity}: pointwise accuracy is essentially flat over a $20\times$ range of $\lambda_{\mathrm{NO}}$, consistent with a scalar-only rescaling of the per-mode score magnitudes.

\paragraph{Regime~III ($\nu \ll 1$, $\zeta \gg 1$): NO very precise.}
This requires $\gamma \ll \nu \ll 1$: the surrogate is extremely accurate at this mode. The posterior gap is bounded by $\delta_{\mathrm{post}} \leq \gamma P_{\mathbf{u}}$, independent of $\sigma_\tau$, and the relative score error is bounded by a constant ($\leq 2\gamma/\nu + 2$). One cannot enter Regime~III without simultaneously providing the bound that controls the error.

\paragraph{Regime~IV ($\nu \sim 1$): crossover.} All quantities are finite and the distribution-agnostic bounds yield computable constants depending on $\gamma$ and $|H|^2$. No singularity occurs at the transition boundary.

Remarkably, the Gaussian moment-matching approximation is benign in all four regimes: where the guidance matters (Regimes~II and~III), the spectral profile of $\tilde{\lambda}_\tau(k)$ is preserved regardless of distributional assumptions, and where the approximation error is least controlled in relative terms (Regime~II), the residual is absorbed by a single scalar
hyperparameter.

\begin{table}[t]
\centering
\small
\renewcommand{\arraystretch}{1.4}
\begin{tabular}{@{}lcccc@{}}
\toprule
\textbf{Regime} &
\textbf{Condition} &
$\mathbb{E}[|\epsilon|^2]$ &
$\dfrac{\mathbb{E}[|\epsilon|^2]}
       {\mathbb{E}[|s^{\mathrm{approx}}|^2]}$ &
\textbf{Mechanism} \\
\midrule
I   & $\nu \gg 1$
    & $O(\nu^{-2})$
    & $O(1)$
    & Score \& error vanish \\[3pt]
II  & $\nu \ll 1$, $\zeta \ll 1$
    & $O(1/(P\nu))$
    & $O(\gamma/\nu)$
    & Profile preserved \\[3pt]
III & $\nu \ll 1$, $\zeta \gg 1$
    & $O(1/(P\nu))$
    & $O(1)$
    & NO precision \\[3pt]
IV  & $\nu \sim 1$
    & $O(1/P)$
    & $O(1)$
    & Transient \\
\bottomrule
\end{tabular}
\caption{Scaling of the score error $\epsilon$ across the four identified regimes. All bounds are distribution-free, requiring only the calibrated second-order statistics.}
\label{tab:regime-summary}

\end{table}

\begin{figure}[t]
  \centering
  \includegraphics[width=0.8\textwidth]{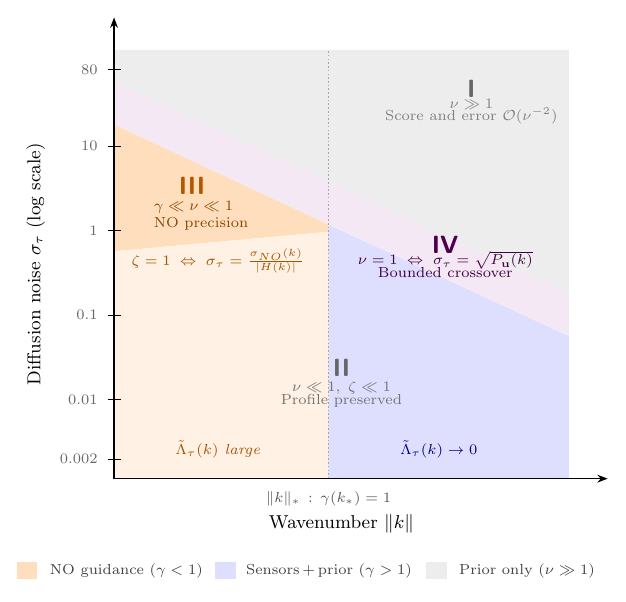}
    \caption{\textbf{Regime decomposition of the wavenumber--diffusion-noise plane.} The Gaussian moment-matching approximation is partitioned into four regimes by three boundaries derived from the calibrated spectral quantities: $\nu = 1 \Leftrightarrow \sigma_\tau = \sqrt{P_{\mathbf{u}}(k)}$, $\zeta = 1 \Leftrightarrow \sigma_\tau = \sigma_{\mathrm{NO}}(k)/|H(k)|$, and $\gamma(k_\ast) = 1$, marking the wavenumber above which the surrogate residual dominates the signal and the spectral weight $\tilde{\lambda}_\tau(k)$ tends to 0. Scaling of the score errors across regimes is summarized in Table~\ref{tab:regime-summary}; see Sec.~\ref{sec:regime-analysis} for the per-regime analysis.}
  \label{fig:regime-diagram}

\end{figure}

\subsection{Limitations and future work}
\label{sec:limitations}

\paragraph{Out-of-distribution calibration.}
The spectral quantities $H(k)$, $\sigma^2_{\mathrm{NO}}(k)$, and $P_{\mathbf{u}}(k)$ are estimated from a held-out split that shares the same geological and source distributions as the training data.
Out-of-distribution geological structures could shift the $\gamma(k) = 1$ crossover that separates guidance-active from guidance-suppressed
frequencies, degrading the spectral calibration precisely in the Regime~II--III transition where the guidance is most consequential.
The sensor term provides a partial safeguard by anchoring the reconstruction at observed locations, but it cannot compensate for systematic miscalibration of the spectral model in unobserved regions.
Monitoring the spectral residual statistics at test time and flagging samples whose empirical $\gamma(k)$ profile deviates from the calibration tables is a natural diagnostic that we leave to future work.

\paragraph{Scope of the spectral observation model.}
The spectral calibration requires the surrogate residual is approximately wide-sense stationary in the Fourier basis, characterized by mode-dependent variance $\sigma^2_{\mathrm{NO}}(k)$ and a real-valued transfer function $H(k)$. This structure is empirically well-justified for the FNO family surrogate used here (Appendix~\ref{app:spectral-empirical}), where the dominant error mechanism, spectral truncation, is a shift-invariant global filter by construction. Surrogates with qualitatively different error structure (e.g., transformer-based operators with localized attention errors, or graph-based operators on irregular meshes) may require a different observation model. The same LMMSE framework applies whenever an approximately diagonal residual
covariance can be estimated, but verifying this assumption for a new surrogate is a prerequisite to using the calibrated guidance.

\paragraph{Sampling cost.}
Iterative diffusion sampling is substantially slower than a single surrogate forward pass: generating one posterior sample requires $64$ sequential denoiser evaluations, each involving a forward and backward pass through the network. 
The spectral NO score itself is negligible in cost (two FFTs per step), so the bottleneck is entirely the denoiser and the sensor-term VJP. 
For applications requiring real-time predictions, consistency distillation or few-step amortized samplers that reduce the number of denoiser calls are promising directions; the closed-form spectral score would transfer directly to any such accelerated sampler since it does not depend on the denoiser architecture.

\paragraph{Synthetic sensor data.}
Most significantly for practical deployment, the current method uses sensor observations generated from SEM simulations rather than physical instruments. 
In a real-world scenario, sensor data would introduce distribution shifts between the synthetic training distribution and the actual measurements that are not captured by the isotropic noise model
$\boldsymbol{\varepsilon} \sim \mathcal{N}(\mathbf{0},
\sigma_y^2 I)$.
The present work establishes the methodological foundation under controlled conditions; extending the framework to real sensor data, by incorporating structured observation noise models or domain adaptation from synthetic to real distributions, is the most important next step toward operational deployment.

\section{Conclusion}
\label{sec:conclusion}
We introduced a spectrally shaped likelihood for integrating neural operator predictions into diffusion posterior sampling, derived from LMMSE marginalization over the clean Fourier coefficients.
The resulting guidance score admits a closed form expression that accounts for the accuracy of the surrogate, requires only two FFTs per reverse diffusion step, and involves no backpropagation through the denoiser.
An exact score identity (Proposition~\ref{prop:exact-score-main})
reveals that the NO likelihood score measures the shift in the optimal clean state estimate upon conditioning on the surrogate observation, and a distribution-free regime analysis shows that frequency dependence of the guidance is preserved regardless of distributional assumptions on the clean signal or the surrogate residual.
On three-dimensional elastic wavefield reconstruction, the method achieves near-zero spectral bias at both $5\%$ and $2\%$ sensor coverage, where isotropic surrogate guidance reimports the surrogate's spectral bias nearly intact which confirms that frequency-dependent calibration is essential, not merely beneficial.
The conditions under which the method applies are precisely
characterized: any surrogate whose residual covariance is approximately diagonal in the Fourier basis admits the calibrated
guidance score, and the empirical coherence diagnostic of
Appendix~\ref{app:wss-verification} provides a prerequisite check
for new surrogates. The same approach therefore applies in principle whenever a fast surrogate exhibits structured spectral error that can be calibrated against high-fidelity reference data, a setting common to FNO-family surrogates across wave propagation, fluid dynamics, and other PDE applications.

\section*{Acknowledgments}
This research was supported by the NVIDIA Academic Grant Program using A100 GPU-Hours (project NEUROELASTOSIM). This work was granted access to the HPC resources of IDRIS under the allocations 2026-AD011017607, 2026-AD011017530R1, 2025-AD011015929, made by GENCI. This work was performed using computational resources from the ``Mésocentre'' computing center of Université Paris-Saclay, CentraleSupélec and École Normale Supérieure Paris-Saclay supported by CNRS and Région Île-de-France (\url{https://mesocentre.universite-paris-saclay.fr/}). Contributions of Fanny Lehmann were primarily supported by the ETH AI Center through their postdoctoral fellowship.

{

\bibliographystyle{plainnat}
\bibliography{references}


\appendix

\section{Score factorization at the noisy state}
\label{app:score-factorization}
This appendix provides the full derivation of the approximate posterior score decomposition~\eqref{eq:posterior-score-three-terms} used in the main text.
We first establish the conditional independence of the observation channels at the clean wavefield (App.~\ref{app:cond-indep}), then derive the exact joint score at the noisy state (App.~\ref{app:ex-joint-score}), and finally state the modeling choice~\eqref{eq:noisy-factor} that yields the three-term decomposition (App.~\ref{app:modeling-choice}).

\subsection{Conditional independence of the observation channels}
\label{app:cond-indep}

We prove the factorization~\eqref{eq:cond-indep} used in the
main text.
Let $\boldsymbol{\xi} = (a, \mathbf{x}_s, \boldsymbol{\theta}_s)$
collect all geological and source parameters, and write the
generative model as
$\boldsymbol{\xi} \xrightarrow{\text{PDE}} \mathbf{u}
\xrightarrow{\mathcal{M}_\mathcal{S} + \varepsilon} \mathbf{y}$
and $\boldsymbol{\xi} \xrightarrow{G_\phi} \mathbf{u}_\mathrm{NO}$.
By the chain rule,
\begin{equation}
  p(\mathbf{y}, \mathbf{u}_\mathrm{NO} \mid \mathbf{u})
  \;=\;
  \int
  p(\mathbf{y} \mid \mathbf{u}_\mathrm{NO}, \boldsymbol{\xi},
    \mathbf{u})\;
  p(\mathbf{u}_\mathrm{NO} \mid \boldsymbol{\xi}, \mathbf{u})\;
  p(\boldsymbol{\xi} \mid \mathbf{u})\;
  d\boldsymbol{\xi}.
\end{equation}
Two observations simplify this expression:
\begin{enumerate}
  \item $p(\mathbf{y} \mid \mathbf{u}_\mathrm{NO},
    \boldsymbol{\xi}, \mathbf{u}) = p(\mathbf{y} \mid \mathbf{u})$,
    because $\mathbf{y} = \mathcal{M}_\mathcal{S}(\mathbf{u}) +
    \boldsymbol{\varepsilon}$ and the instrument noise
    $\boldsymbol{\varepsilon}$ is independent of
    $(\boldsymbol{\xi}, \mathbf{u}_\mathrm{NO})$.
  \item $p(\mathbf{u}_\mathrm{NO} \mid \boldsymbol{\xi},
    \mathbf{u}) = p(\mathbf{u}_\mathrm{NO} \mid
    \boldsymbol{\xi})$, because $\mathbf{u}_\mathrm{NO} =
    G_\phi(\boldsymbol{\xi})$ is a deterministic function of
    $\boldsymbol{\xi}$ alone.
\end{enumerate}
Substituting and pulling the $\boldsymbol{\xi}$-independent
factor out of the integral:
\begin{equation}
  p(\mathbf{y}, \mathbf{u}_\mathrm{NO} \mid \mathbf{u})
  \;=\;
  p(\mathbf{y} \mid \mathbf{u})
  \int p(\mathbf{u}_\mathrm{NO} \mid \boldsymbol{\xi})\;
  p(\boldsymbol{\xi} \mid \mathbf{u})\;d\boldsymbol{\xi}
  \;=\;
  p(\mathbf{y} \mid \mathbf{u})\;
  p(\mathbf{u}_\mathrm{NO} \mid \mathbf{u}).
\end{equation}
Note that the inverse problem
$\mathbf{u} \mapsto \boldsymbol{\xi}$ may be ill-posed (many
parameter configurations can produce the same wavefield), so
$p(\mathbf{u}_\mathrm{NO} \mid \mathbf{u})$ is a non-trivial
distribution; the factorization holds nonetheless because the
instrument noise $\boldsymbol{\varepsilon}$ carries no information
about which $\boldsymbol{\xi}$ generated $\mathbf{u}$.

\subsection{Exact joint score at the noisy state}
\label{app:ex-joint-score}
We seek to sample from $p_\tau(\mathbf{u}_\tau \mid \mathbf{y}, \mathbf{\uNO})$.
By Bayes' rule,
\begin{equation}
  \label{eq:app-bayes-time}
  p_\tau(\mathbf{u}_\tau \mid \mathbf{y}, \mathbf{\uNO})
  \;\propto\;
  p_\tau(\mathbf{u}_\tau)\;
  p(\mathbf{y}, \mathbf{\uNO} \mid \mathbf{u}_\tau),
\end{equation}
so the posterior score decomposes as
\begin{equation}
  \label{eq:app-score-bayes}
  \nabla_{\mathbf{u}_\tau}\log p_\tau(\mathbf{u}_\tau \mid
    \mathbf{y}, \mathbf{\uNO})
  \;=\;
  \nabla_{\mathbf{u}_\tau}\log p_\tau(\mathbf{u}_\tau)
  \;+\;
  \nabla_{\mathbf{u}_\tau}\log p(\mathbf{y}, \mathbf{\uNO}
    \mid \mathbf{u}_\tau).
\end{equation}
The sparse observation model (Sec.~\ref{sec:obs-model}) specifies
$p(\mathbf{y} \mid \mathbf{u})$ for the clean wavefield only.
Since the measurement and NO likelihoods are defined for the clean
wavefield $\mathbf{u}$, and $(\mathbf{y},\mathbf{\uNO}) \perp \mathbf{u}_\tau \mid \mathbf{u}$ (the diffusion noise $\boldsymbol{\eta}_\tau$ is independent of both observation channels), the joint likelihood at the noisy state is obtained by marginalizing over $\mathbf{u}$:
\begin{equation}
  \label{eq:app-joint-marginal}
  p(\mathbf{y}, \mathbf{\uNO} \mid \mathbf{u}_\tau)
  \;=\;
  \int p(\mathbf{y}, \mathbf{\uNO} \mid \mathbf{u})\;
  p(\mathbf{u} \mid \mathbf{u}_\tau)\,d\mathbf{u}.
\end{equation}
Substituting the conditional independence assumption~\eqref{eq:cond-indep}:
\begin{equation}
  \label{eq:app-joint-factored}
  p(\mathbf{y}, \mathbf{\uNO} \mid \mathbf{u}_\tau)
  \;=\;
  \int p(\mathbf{y} \mid \mathbf{u})\;
  p_{\mathrm{NO}}(\mathbf{\uNO} \mid \mathbf{u})\;
  p(\mathbf{u} \mid \mathbf{u}_\tau)\,d\mathbf{u}.
\end{equation}

Although $\mathbf{y} \perp \mathbf{\uNO} \mid \mathbf{u}$, conditioning on the noisy state $\mathbf{u}_\tau$ leaves residual uncertainty about $\mathbf{u}$, through which $\mathbf{y}$ and $\mathbf{\uNO}$ remain coupled.
Consequently, the integral in~\eqref{eq:app-joint-factored} does not factorize into a product of marginals over the noisy state, and the exact joint likelihood score reads
\begin{equation}
  \label{eq:app-exact-joint-score}
  \nabla_{\mathbf{u}_\tau}\log p(\mathbf{y}, \mathbf{\uNO}
    \mid \mathbf{u}_\tau)
  \;=\;
  \nabla_{\mathbf{u}_\tau}\log p(\mathbf{y} \mid \mathbf{u}_\tau)
  \;+\;
  \nabla_{\mathbf{u}_\tau}\log p(\mathbf{\uNO} \mid \mathbf{y},\,
    \mathbf{u}_\tau),
\end{equation}
where the second term conditions on $\mathbf{y}$.

\subsection{Modeling choice}
\label{app:modeling-choice}
The exact joint score~\eqref{eq:app-exact-joint-score} requires evaluating $\nabla_{\mathbf{u}_\tau}\log p(\mathbf{\uNO} \mid \mathbf{y}, \mathbf{u}_\tau)$, the surrogate likelihood score conditioned on the sensor observations.
This quantity has no closed form: it depends on the data prior through the noisy-state posterior $p(\mathbf{u} \mid \mathbf{u}_\tau, \mathbf{y})$, which is precisely what diffusion posterior sampling is approximating in the first place.
We therefore adopt the structural modeling choice
\begin{equation}
  \label{eq:app-modelling-choice}
  p(\mathbf{\uNO} \mid \mathbf{y}, \mathbf{u}_\tau)
  \;\approx\;
  p(\mathbf{\uNO} \mid \mathbf{u}_\tau),
\end{equation}
which, combined with~\eqref{eq:app-exact-joint-score} and~\eqref{eq:app-score-bayes}, yields the three-term decomposition~\eqref{eq:posterior-score-three-terms} of the main text.

The choice is consistent with the approximation that DPS itself already makes. Standard DPS replaces the noisy-state posterior with a point mass at the Tweedie estimate, $p(\mathbf{u} \mid \mathbf{u}_\tau) \approx \delta(\mathbf{u} - \bar{\mathbf{u}}_0(\mathbf{u}_\tau))$. Under this collapse, both likelihoods depend on $\mathbf{u}_\tau$ only through the deterministic point estimate $\bar{\mathbf{u}}_0(\mathbf{u}_\tau, \tau)$, and~\eqref{eq:app-modelling-choice} holds exactly: conditioning on $\mathbf{y}$ provides no additional information about $\mathbf{\uNO}$ beyond what $\bar{\mathbf{u}}_0$ already determines. Our use of an LMMSE refinement in the NO channel (Sec.~\ref{sec:no-guided-dps}) sharpens the per-mode posterior variance but acts within the NO channel alone, introducing no coupling to the sensor channel at the noisy state. The factorization~\eqref{eq:app-modelling-choice} is thus the same mild modeling assumption that DPS already relies on, applied independently to each observation channel.

\section{Spectral observation model: notation and estimation}
\label{app:spectral-validation}

This appendix provides the full notation for the spectral observation model introduced in Sec.~\ref{sec:no-guided-dps}, the nonparametric estimators used to calibrate the mode-dependent quantities. 

\subsection{Fourier-domain notation}
\label{app:dft-notation}

Let $\mathcal{F}$ denote the unitary discrete Fourier transform applied jointly over the spatial and temporal axes $(N_x, N_y, T)$, independently per channel:
\begin{equation}
  \label{eq:app-dft-def}
  \mathcal{F}:\; \R^{C \times N_x \times N_y \times T}
  \;\to\; \mathbb{C}^{C \times N_x \times N_y \times T},
\end{equation}
with unitarity ensuring $\mathcal{F}^\dagger = \mathcal{F}^{-1}$ and $\|\mathcal{F}(\mathbf{u})\| = \|\mathbf{u}\|$.
We index frequency modes by the multi-index $k = (k_x, k_y, k_t) \in \K \coloneqq \mathbb{Z}_{N_x} \times \mathbb{Z}_{N_y} \times \mathbb{Z}_{T}$.
Throughout, we use $\hat{\cdot}$ to denote Fourier-domain quantities:
$\hat{\mathbf{v}} \coloneqq \mathcal{F}(\mathbf{v})$ and
$\hat{v}(k) \coloneqq \mathcal{F}(\mathbf{v})(k)$ for the coefficient at mode $k$.

Since the spectral observation model and all subsequent derivations operate independently at each mode $k$, the analysis is identical to the one-dimensional case; no structure specific to the three-dimensional grid is required beyond the definition of $\K$.

\subsection{Per-mode observation model}
\label{app:per-mode-model}

We model the spectral relationship between the neural operator
prediction and the ground truth via a channel- and mode-dependent
transfer function $H: \{1,\dots,C\} \times \K \to \mathbb{R}$,
acting element-wise in Fourier space.
For clarity, we present the model for a single channel; the extension to $C$ channels is immediate since all quantities factorize independently across channels.
At each frequency mode $k$, the NO prediction is modeled as
\begin{equation}
  \label{eq:app-per-mode}
  \mathcal{F}(\mathbf{\uNO})(k)
  \;=\;
  H(k)\,\mathcal{F}(\mathbf{u})(k)
  \;+\;
  \hat{\eta}_{\mathrm{NO}}(k),
\end{equation}
where $H(k)$ absorbs all \emph{systematic} errors of the neural
operator at mode $k$ so that the residual $\hat{\eta}_{\mathrm{NO}}(k)$ captures only the remaining stochastic fluctuations, modeled as zero-mean with variance
$\sigma^2_{\mathrm{NO}}(k)$; no assumption is made on the
distributional form of $\hat{\eta}_{\mathrm{NO}}$.

\subsection{Wide-sense stationarity assumption}
\label{app:wss}

We assume the spatio-temporal residual error field
$\mathbf{\eta}_{NO} \coloneqq \mathbf{\uNO} - \mathcal{F}^{-1}(H \odot \mathcal{F}(\mathbf{u}))$ is wide-sense stationary (WSS) in $(x,y,t)$ over the training ensemble of geologies and source locations.
While an individual elastodynamic wavefield is transient and spatially localized, averaging over uniformly distributed source locations and varied heterogeneous geologies yields an ensemble error field that is approximately translation-invariant.
This is further supported by the FNO architecture, whose dominant
error mechanism (spectral truncation at high wavenumbers) is a
shift-invariant global filter by construction.

By the Wiener--Khinchin theorem, WSS implies that distinct Fourier modes of the residual error are uncorrelated:
\begin{equation}
  \label{eq:app-wss}
  \mathbb{E}\!\left[\hat{\eta}_{NO}(k)\,\overline{\hat{\eta}_{NO}(k')}\right]
  \;=\;
  \sigma^2_{\mathrm{NO}}(k)\,\delta_{kk'},
\end{equation}
where each channel is allowed its own variance profile, while
cross-channel residuals remain uncorrelated.
The covariance $\Sigma_{\mathrm{NO}}$ is therefore diagonal in the Fourier basis, with entries $\sigma^2_{\mathrm{NO}}(k) \in \R_{>0}$.

The residual is therefore characterized by
\begin{equation}
  \mathbb{E}[\hat{\eta}_{\mathrm{NO}}(k)] = 0,
  \qquad
  \mathbb{E}\!\left[|\hat{\eta}_{\mathrm{NO}}(k)|^2\right]
  = \sigma^2_{\mathrm{NO}}(k),
\end{equation}
with no further distributional assumption.
The conditional first two moments of the NO prediction at each
mode are
\begin{equation}
  \mathbb{E}\!\left[
    \mathcal{F}(\mathbf{\uNO})(k) \;\middle|\;
    \mathcal{F}(\mathbf{u})(k)
  \right]
  = H(k)\,\mathcal{F}(\mathbf{u})(k),
  \qquad
  \mathrm{Var}\!\left[
    \mathcal{F}(\mathbf{\uNO})(k) \;\middle|\;
    \mathcal{F}(\mathbf{u})(k)
  \right]
  = \sigma^2_{\mathrm{NO}}(k).
\end{equation}
The joint distribution over all modes and channels factorizes as a product of scalar complex Gaussians.

\subsection{Estimation of spectral quantities}
\label{app:spectral-estimation}

The spectral model introduces three mode-dependent quantities that must be estimated before inference.
We compute all of them from $N$ paired ground-truth/MIFNO samples on a held-out calibration split and freeze them as lookup tables over
$\{1,\dots,C\} \times \K$.

\paragraph{Signal power spectrum.}
The ensemble-averaged signal power at each channel and mode is
\begin{equation}
  \label{eq:app-Pu}
  P_{\mathbf{u}}(k)
  \;=\;
  \frac{1}{N}\sum_{n=1}^{N}
  \left|\mathcal{F}(\mathbf{u}^{(n)})(k)\right|^2.
\end{equation}

\paragraph{Transfer function.}
The transfer function $H(k)$ is identified as the minimum
mean-square-error linear predictor of $\mathcal{F}(\mathbf{\uNO})(k)$ from $\mathcal{F}(\mathbf{u})(k)$, giving the cross-spectral estimator
\begin{equation}
  \label{eq:app-H}
  \tilde{H}(k)
  \;=\;
  \frac{
    \dfrac{1}{N}\displaystyle\sum_{n=1}^{N}
    \mathcal{F}(\mathbf{\uNO}^{(n)})(k)\;
    \overline{\mathcal{F}(\mathbf{u}^{(n)})(k)}
  }{
    P_{\mathbf{u}}(k)
  } \;\in\; \mathbb{C}.
\end{equation}
We empirically observe $|\arg \tilde{H}(k)| \ll 1$ across the full spectrum (Fig.~\ref{fig:transfer-function}), consistent with the shift-invariant low-pass nature of FNO truncation errors, and accordingly take $H(k) \coloneqq \mathrm{Re}\,\tilde{H}(k)$ in all 
subsequent expressions.

\paragraph{Residual variance.}
The residual variance at each channel and mode is estimated as
\begin{equation}
  \label{eq:app-sigma-no}
  \sigma^2_{\mathrm{NO}}(k)
  \;=\;
  \frac{1}{N}\sum_{n=1}^{N}
  \left|
    \mathcal{F}(\mathbf{\uNO}^{(n)})(k)
    - H(k)\,\mathcal{F}(\mathbf{u}^{(n)})(k)
  \right|^2.
\end{equation}

\paragraph{Relative error power ratio.}
For reference, we define the dimensionless noise-to-signal ratio
\begin{equation}
  \label{eq:app-gamma}
  \gamma(k)
  \;\coloneqq\;
  \frac{\sigma^2_{\mathrm{NO}}(k)}{|H(k)|^2\,P_{\mathbf{u}}(k)},
\end{equation}
which captures the relative accuracy of the NO at each frequency,
independent of the absolute signal amplitude.
This quantity governs the regime analysis of the Gaussian marginal approximation (Appendix~\ref{app:gaussian-justification}).

\subsection{Empirical validation}
\label{app:spectral-empirical}

Figure~\ref{fig:transfer-function} shows the transfer function
$\tilde{H}(k)$ estimated from $N=2{,}000$ paired ground-truth/MIFNO
samples on the held-out calibration split.
The magnitude $|\tilde{H}(k)|$ (left) confirms the expected
spectral-bias profile: $|H| \approx 0.9$ at low wavenumbers and
decreasing monotonically toward zero at high $\|k\|$, with all
three components behaving similarly.
The phase $\arg \tilde{H}(k)$ (right) remains bounded within
$\pm 0.5$\,rad across the spectrum where $|H(k)|$ is non-negligible,
justifying the real-valued reduction $H = \mathrm{Re}\,\tilde{H}$
used in all subsequent expressions.

\begin{figure}[h]
  \centering
  \includegraphics[width=0.8\textwidth]{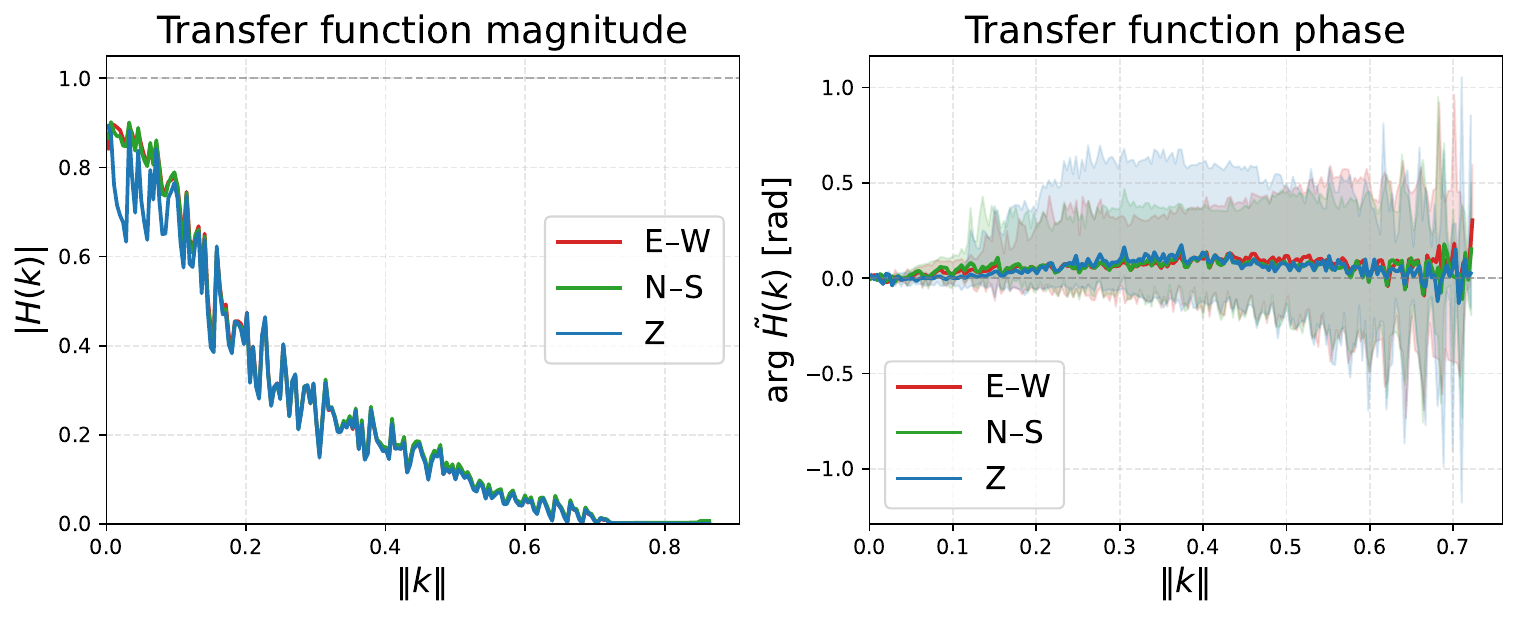}
  \caption{Transfer function $\tilde{H}(k)$ estimated from     the calibration split, radially binned over $\|k\|$.
    Left: magnitude $|\tilde{H}(k)|$ decreases from $\approx 0.9$ at low $\|k\|$ to near zero at high $\|k\|$, confirming the expected spectral-bias profile.
    Right: phase $\arg \tilde{H}(k)$ (median and 5th--95th
    percentile band), shown only for modes with $|H(k)| > 0.01$.
    The phase is bounded within $\pm 0.5$\,rad, justifying the real-valued reduction $H = \mathrm{Re}\,\tilde{H}$.}
  \label{fig:transfer-function}
\end{figure}

Figure~\ref{fig:spectral-error} shows the NO error landscape that motivates the spectral shaping.
The residual variance $\sigma^2_{\mathrm{NO}}(k)$ and signal power $P_{\mathbf{u}}(k)$ both decay with $\|k\|$, but at different rates: the signal power spans approximately seven orders of magnitude, while the residual variance spans eight.
Their ratio, the relative NO error $\gamma(k) = \sigma^2_{\mathrm{NO}}(k) / (|H(k)|^2 P_{\mathbf{u}}(k))$,
crosses unity near $\|k\| \approx 0.2$; below this threshold the NO prediction is more accurate than the prior ($\gamma < 1$), while above it the NO residual dominates the signal ($\gamma \gg 1$).
This crossover determines the spectral boundary beyond which NO guidance should be suppressed.

\begin{figure}[h]
  \centering
  \includegraphics[width=\textwidth]{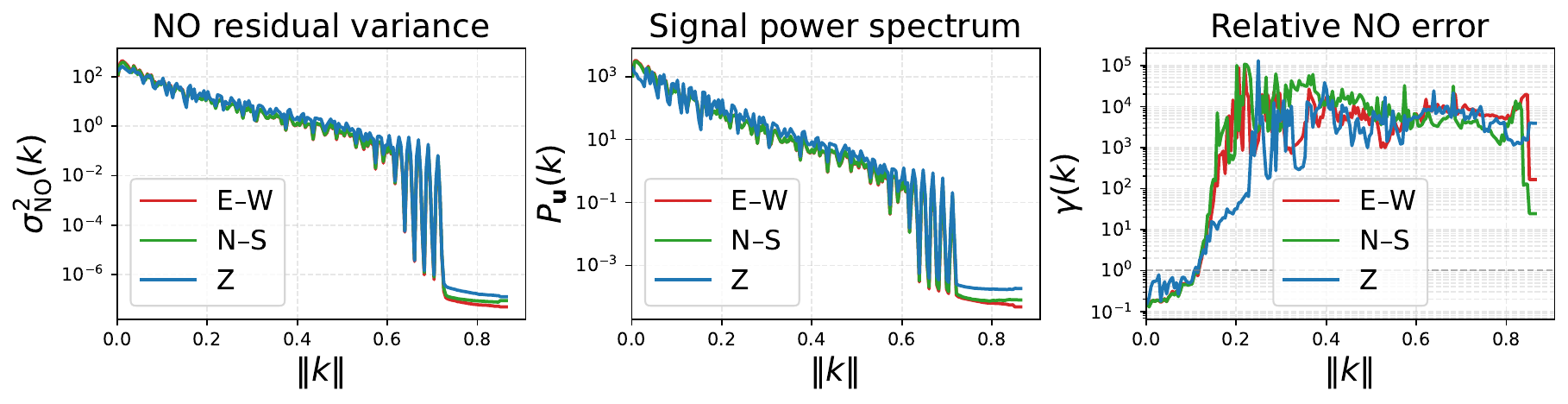}
  \caption{NO error profile estimated from the calibration     split ($N = 2{,}000$ paired samples), radially binned      over $\|k\|$.
    From left to right: NO residual variance $\sigma^2_{\mathrm{NO}}(k)$, signal power spectrum
    $P_{\mathbf{u}}(k)$, and relative NO error $\gamma(k) = \sigma^2_{\mathrm{NO}}(k) / (|H(k)|^2 P_{\mathbf{u}}(k))$. The crossover $\gamma = 1$ (dashed grey) occurs near $\|k\| \approx 0.2$; below it the NO is more accurate than the prior, above it the residual dominates.}
  \label{fig:spectral-error}
\end{figure}

The expected NO guidance magnitudes shown in
Fig.~\ref{fig:expected-guidance} are derived as follows.
For the spectrally shaped score, the expected squared
magnitude at each mode is
\begin{equation}
  \mathbb{E}\!\left[|\tilde{\lambda}_\tau(k)
  \cdot r(k)|^2\right]
  \;=\;
  \frac{4\,|H(k)|^2\,\alpha(k)^2}{\lambda_\tau(k)},
\end{equation}
where $\lambda_\tau(k) = \sigma^2_{\mathrm{NO}}(k)
+ |H(k)|^2 \sigma^2_\tau \alpha(k)$.
For the isotropic baseline ($H(k)=1$,
$\sigma^2_{\mathrm{NO}}(k) = \sigma^2_{\mathrm{NO,iso}}$,
$\alpha(k) = 1$), this reduces to
$4/(\sigma^2_{\mathrm{NO,iso}} + \sigma^2_\tau)$.
The upturn of the spectrally shaped curves at high $\|k\|$
for $\sigma_\tau \leq 0.01$ is not consequential: the ODE
drift coefficient $c(\tau) = s^2 \sigma^2 (\dot{\sigma}/\sigma)$
vanishes as $\sigma_\tau \to 0$, so the NO guidance
contribution to the reverse trajectory is negligible at the
final steps regardless of the score magnitude.

\subsection{Empirical assessment of the diagonal-covariance approximation}
\label{app:wss-verification}

The spectral observation model treats the residual covariance as diagonal in the Fourier basis (Sec.~\ref{app:wss}). Strict WSS would imply exact diagonality; we expect only an approximate version to 
hold on a finite, transient wavefield, and we quantify the departure empirically.
We verify this empirically by estimating the off-diagonal
cross-spectral coherence from the same $N=2{,}000$-sample
calibration split used for the spectral model.

\paragraph{Test statistic.}
For a pair of distinct modes $(k, k')$ with $k \neq k'$, the sample cross-spectrum is
\begin{equation}
  \label{eq:app-cross-spec}
  \hat{C}(k, k')
  \;\coloneqq\;
  \frac{1}{N}\sum_{n=1}^{N}
  \hat{\eta}_{\mathrm{NO}}^{(n)}(k)\;
  \overline{\hat{\eta}_{\mathrm{NO}}^{(n)}(k')},
\end{equation}
where $\hat{\eta}_{\mathrm{NO}}^{(n)}(k)
= \mathcal{F}(\mathbf{u}_{\mathrm{NO}}^{(n)})(k)
- H(k)\,\mathcal{F}(\mathbf{u}^{(n)})(k)$ is the Fourier-domain residual at mode $k$ for sample $n$.
We normalize this into a dimensionless coherence
\begin{equation}
  \label{eq:app-coherence}
  \mathrm{coh}(k, k')
  \;\coloneqq\;
  \frac{|\hat{C}(k, k')|}
       {\sqrt{\sigma^2_{\mathrm{NO}}(k)\;
              \sigma^2_{\mathrm{NO}}(k')}},
\end{equation}
which lies in $[0, 1]$: zero indicates uncorrelated modes (WSS satisfied) and unity indicates perfect correlation (WSS violated).
Under the null hypothesis of exact WSS, the expected coherence from finite-sample noise is $\mathbb{E}[\mathrm{coh}] \approx \sqrt{\pi/(4N)} \approx 0.02$ for $N = 2{,}000$.
\paragraph{Sampling procedure.}
The full cross-spectral matrix has $\sim\!1.6 \times 10^5$ modes per channel, making exhaustive computation infeasible.
We instead draw $P = 400{,}000$ random pairs of distinct mode indices from the rFFT grid
$(N_x \times N_y \times N_f) = (32 \times 32 \times 161)$ and compute~\eqref{eq:app-cross-spec} via a streaming pass over the calibration set.
For each pair, we record the mode separation $\|k - k'\|$ in normalised frequency units.
To assess whether any departures concentrate in specific spectral regions, we additionally stratify the pairs by the wavenumber magnitude of both modes:
\emph{high--high} (both $\|k\| \geq 0.15$) and
\emph{mixed} (one below, one above the threshold), where
$\|k\| = 0.15$ corresponds approximately to the crossover
$\gamma(k) \approx 1$ (Fig.~\ref{fig:spectral-error}, right).

\paragraph{Results.}
Figure~\ref{fig:wss-coherence} (left) shows the mean and 95th
percentile of the coherence as a function of mode separation, pooled over all pairs.
For separations $\|k - k'\| > 0.2$, the mean coherence stabilizes at $\approx\!0.06$--$0.10$, roughly $3$--$4\times$ the finite-sample noise floor.
A narrow spike of elevated coherence appears at
$\|k - k'\| < 0.1$, indicating that immediately neighboring
Fourier modes are moderately correlated; this is consistent with spectral leakage on the finite discrete grid and/or residual non-stationarity from the transient, finite-support nature of individual wavefields.

Figure~\ref{fig:wss-coherence} (right) isolates the mixed stratum (one low-frequency mode, one high-frequency mode).
The mean coherence drops to $\approx\!0.03$--$0.05$, close to the noise floor.
This confirms that cross-regime coupling, between the frequencies where NO guidance is active ($\|k\|$ small, $\gamma < 1$) and those where it is suppressed ($\|k\|$ large, $\gamma \gg 1$), is negligible.

\paragraph{Implications for the spectral guidance.}
The residual exhibits moderate near-diagonal correlation (coherence $\sim\!0.06$–$0.10$ at small mode separations, decaying to the permutation-test noise floor at separations $\|k - k'\| > 0.2$). 
This means the residual covariance is approximately, but not exactly, diagonal: WSS holds in a weak sense, with a correlation length $\ell_{\mathrm{c}} \lesssim 0.1$ in normalized frequency.

This residual structure does not compromise the per-mode 
factorization used in the guidance score, for two reasons.
\emph{First,} the spectral weight $\tilde\lambda_\tau(k)$ varies smoothly over $\|k\|$ on scales much larger than the correlation length of the residual ($\lesssim 0.1$ in normalized frequency).
When the weight is approximately constant over the correlation bandwidth, treating correlated neighbors as independent produces the same integrated guidance as a properly banded covariance treatment: the per-mode errors are absorbed in aggregate, and the spectral
\emph{profile} of the guidance, the property that distinguishes our method from isotropic alternatives, is preserved.
\emph{Second,} cross-regime coupling between guidance-active modes ($\gamma < 1$) and guidance-suppressed modes ($\gamma \gg 1$) is negligible (Fig.~\ref{fig:wss-coherence}, right), so the residual diagonal structure is preserved precisely in the spectral region where the LMMSE shaping is most consequential.

\begin{figure}[h!]
  \centering
  \includegraphics[width=\textwidth]{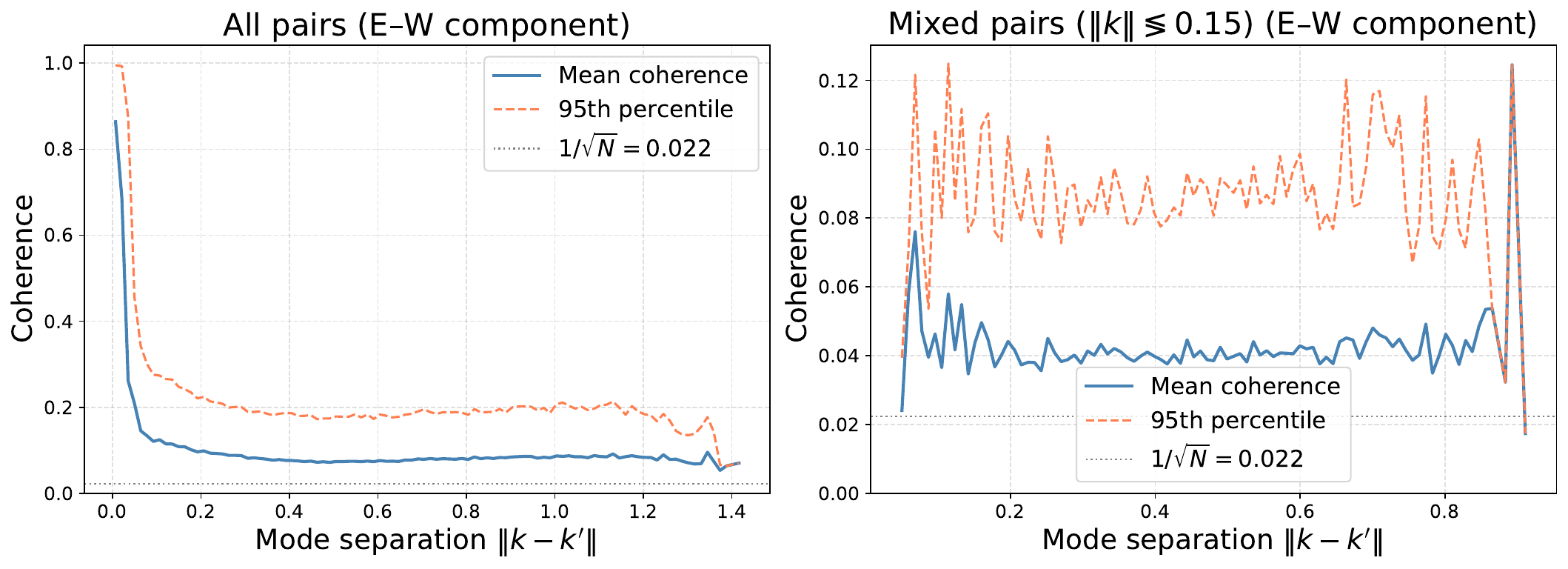}
  \caption{Off-diagonal cross-spectral coherence of      the NO residual (E--W component), estimated from     $400{,}000$ random mode pairs on the calibration     split ($N = 2{,}000$).
    Left: all pairs, binned by mode separation $\|k - k'\|$.
    The mean coherence (solid) settles at $3$-$4\times$ the finite-sample noise floor $1/\sqrt{N} \approx 0.022$ (dotted) beyond $\|k - k'\| > 0.2$; the near-diagonal spike is consistent with spectral leakage on the finite grid.
    Right: mixed pairs (one mode with $\|k\| < 0.15$, one with $\|k\| \geq 0.15$).
    Cross-regime coherence is close to the noise floor, confirming that the per-mode factorization is well-justified in the frequency range where the NO guidance is active.}
  \label{fig:wss-coherence}
\end{figure}

\section{LMMSE derivation and properties}
\label{app:lmmse}

This appendix provides the full derivation of the spectrally shaped diffusion posterior used in Sec.~\ref{sec:no-guided-dps}, along with its optimality properties, connection to the neural denoiser, and equivalence with the Bayesian posterior under stronger hypotheses.
\subsection{Per-mode LMMSE estimator}
\label{app:lmmse-derivation}

The forward diffusion gives, at each mode $k$,
\begin{equation}
  \label{eq:app-fwd-fourier}
  Y(k)
  \;\coloneqq\;
  \mathcal{F}(\mathbf{u}_\tau)(k)
  \;=\;
  X(k) + \hat{\eta}_\tau(k),
  \qquad
  \hat{\eta}_\tau(k) \sim \mathcal{CN}(0,\;\sigma^2_\tau),
\end{equation}
where $X(k) \coloneqq \mathcal{F}(\mathbf{u})(k)$ and unitarity of $\mathcal{F}$ preserves the noise variance.
We seek the best linear estimator of $X$ from $Y$.
The only required properties of $X$:
\begin{enumerate}
    \item $\mathbb{E}[X(k)] = 0$, which follows from z-score normalization of the training data ($\mathbb{E}_n[\mathbf{u}^{(n)}] = \mathbf{0}$, hence $\mathbb{E}_n[\mathcal{F}(\mathbf{u}^{(n)})(k)] = 0$ by linearity of $\mathcal{F}$);
    \item $\mathbb{E}[|X(k)|^2] = P_{\mathbf{u}}(k)$, the signal power spectrum estimated from training data via~\eqref{eq:app-Pu}.
\end{enumerate}
No assumption is made on the distributional form of $X$.

The \textbf{linear minimum mean-squared-error (LMMSE)} estimator of $X$ given $Y = X + \hat{\eta}_\tau$ is the solution to
\begin{equation}
  \label{eq:app-lmmse-def}
  \hat{X}_L
  \;\coloneqq\;
  \arg\min_{\hat{X} = aY + b}\;
  \mathbb{E}\!\left[|X - \hat{X}|^2\right].
\end{equation}
By the orthogonality principle, the optimal coefficients satisfy
$\mathbb{E}[(X - \hat{X}_L)\,Y^*] = 0$ and $\mathbb{E}[X - \hat{X}_L] = 0$.
Computing the required second-order statistics, $\mathrm{Var}[Y]
= P_{\mathbf{u}} + \sigma^2_\tau$ and $\mathrm{Cov}[X, Y] = P_{\mathbf{u}}$ (by independence of $X$ and $\hat{\eta}_\tau$), yields:
\begin{equation}
  \label{eq:app-alpha}
  \hat{X}_L(k) = \alpha(k)\,Y(k),
  \qquad
  \alpha(k)
  \;\coloneqq\;
  \frac{P_{\mathbf{u}}(k)}{\sigma^2_\tau + P_{\mathbf{u}}(k)}.
\end{equation}
The coefficient $\alpha(k)$ contracts the noisy observation toward zero (the prior mean) with strength determined by the noise-to-signal ratio $\sigma^2_\tau / P_{\mathbf{u}}(k)$.

The associated LMMSE error is:
\begin{equation}
  \label{eq:app-SigmaL}
  \sigma_L^2(k)
  \;\coloneqq\;
  \mathbb{E}\!\left[|X - \hat{X}_L|^2\right]
  \;=\;
  \frac{\sigma^2_\tau\,P_{\mathbf{u}}(k)}
       {\sigma^2_\tau + P_{\mathbf{u}}(k)}
  \;=\;
  \sigma^2_\tau\,\alpha(k).
\end{equation}

\subsection{Limiting behaviour}
\label{app:lmmse-limits}

At low frequencies where $P_{\mathbf{u}}(k) \gg \sigma^2_\tau$:
$\alpha(k) \to 1$ and $\sigma_L^2(k) \to \sigma^2_\tau$, recovering the standard DPS isotropic posterior.

At high frequencies where $P_{\mathbf{u}}(k) \ll \sigma^2_\tau$:
$\alpha(k) \to P_{\mathbf{u}}(k)/\sigma^2_\tau \to 0$ and
$\sigma_L^2(k) \to P_{\mathbf{u}}(k)$.

The estimation error is bounded by the signal power at each mode:
\begin{equation}
  \label{eq:app-sigma-bound}
  \sigma_L^2(k)
  \;\leq\;
  \min\!\big(P_{\mathbf{u}}(k),\;\sigma^2_\tau\big),
\end{equation}
reflecting the elementary fact that one cannot be wrong about signal content that was never present.
This bound holds for any distribution of $X$ with the specified mean and variance; it does not require Gaussianity.

\section{Marginal likelihood derivation and Gaussian approximation}
\label{app:gaussian-justification}

This appendix derives the moment-matched Gaussian marginal likelihood~\eqref{eq:marginal-likelihood} used in the main text,
establishes an exact identity for the NO likelihood score (Proposition~\ref{prop:exact-score-main}), decomposes the score error introduced by the Gaussian approximation, and characterizes the resulting score error across the 
frequency--diffusion-time plane via a distribution-free regime analysis.

\subsection{Closed-form marginal moments}
\label{app:marginal-moments}

Fix a channel $c$ and mode $k$, and write $X \coloneqq \mathcal{F}(\mathbf{u})(k)$, $Y \coloneqq \mathcal{F}(\mathbf{u}_\tau)(k)$, $Z \coloneqq \mathcal{F}(\mathbf{\uNO})(k)$. We seek the first two moments of $Z$ given $Y$, obtained by
marginalizing over $X$.

Since $Z(k) = H(k)\,X(k) + \hat{\eta}_{\mathrm{NO}}(k)$ is affine in $X$, 
and $\hat{\eta}_{\mathrm{NO}}$ is (i) uncorrelated with $X$ by  construction of $H$ as the LMMSE coefficient, and (ii) independent of 
the diffusion noise $\hat{\eta}_\tau$ (the diffusion process and the 
NO evaluation are produced by independent mechanisms), the LMMSE 
framework gives the best linear predictor of $Z$ from $Y$, namely $\hat{Z}_L$, and the 
associated prediction error directly. Computing
$\mathrm{Cov}[Z, Y] = H\,P_{\mathbf{u}}$ (the cross-term 
$\mathbb{E}[\hat{\eta}_{\mathrm{NO}}(k)\,\overline{X(k)}]$ vanishes by 
orthogonality, and 
$\mathbb{E}[\hat{\eta}_{\mathrm{NO}}(k)\,\overline{\hat{\eta}_\tau(k)}]$ 
vanishes by independence of the two noise sources) and 
$\mathrm{Var}[Y] = P_{\mathbf{u}} + \sigma^2_\tau$, the orthogonality 
principle yields the LMMSE estimator $\hat{Z}_L$ and the associated mean square error $\lambda_\tau$, that is:
\begin{align}
  \label{eq:app-marginal-mean}
  \hat{Z}_L
  &= \frac{\mathrm{Cov}[Z,Y]}{\mathrm{Var}[Y]}\,Y
  = H(k)\,\alpha(k)\,Y, \\[6pt]
  \label{eq:app-marginal-var}
  \lambda_\tau(k)
  &\coloneqq \mathbb{E}\!\left[|Z - \hat{Z}_L|^2\right]
  = |H(k)|^2\,\sigma_L^2(k) + \sigma^2_{\mathrm{NO}}(k),
\end{align}
where the cross-term in the variance expansion vanishes because $(X - \alpha Y)$ depends only on $(X, \hat{\eta}_\tau)$ while $\hat{\eta}_{\mathrm{NO}}$ is uncorrelated with $X$ (by LMMSE definition of $H$) and independent of $\hat{\eta}_\tau$.

Approximating the marginal distribution of $Z$ given $Y$ as the
Gaussian matching these moments, one obtains:
\begin{equation}
  \label{eq:app-marginal-result}
  p\!\left(\mathcal{F}(\mathbf{\uNO})(k) \mid \mathbf{u}_\tau\right)
  \;\approx\;
  \mathcal{CN}\!\left(
    H(k)\,\alpha(k)\,\mathcal{F}(\mathbf{u}_\tau)(k),\;\;
    \lambda_\tau(k)
  \right),
\end{equation}
with the marginalized variance
\begin{equation}
  \label{eq:app-Lambda-tau}
  \lambda_\tau(k)
  \;=\;
  \sigma^2_{\mathrm{NO}}(k)
  \;+\;
  |H(k)|^2\,
  \frac{\sigma^2_\tau\,P_{\mathbf{u}}(k)}
       {\sigma^2_\tau + P_{\mathbf{u}}(k)}.
\end{equation}
The joint distribution over all modes and channels factorizes as a product of these univariate complex Gaussians.

\paragraph{Comparison with standard DPS.}
The standard DPS marginalized variance is $\lambda^{\mathrm{DPS}}_\tau(k) = \sigma^2_{\mathrm{NO}}(k) + \sigma^2_\tau|H(k)|^2$, obtained by replacing $\sigma_L^2(k)$ with $\sigma^2_\tau$.
At low frequencies where $P_{\mathbf{u}}(k) \gg \sigma^2_\tau$, the
two agree.
At high frequencies where $P_{\mathbf{u}}(k) \ll \sigma^2_\tau$,
$\sigma_L^2(k) \to P_{\mathbf{u}}(k) \ll \sigma^2_\tau$ and the
corrected variance is significantly smaller.
However, this reduction is offset by the factor $\alpha(k) \to 0$ in the score numerator (derived below), which suppresses the guidance at precisely these frequencies.
The net effect is a spectral weight that decreases with $\|k\|$,
correctly reflecting the NO's decreasing reliability at high
frequencies.

\subsection{Proof of Proposition~\ref{prop:exact-score-main}}
\label{app:exact-score}
We recall the two observation channels at a single mode $k$ (suppressing channel and mode indices when unambiguous):
\begin{alignat}{2}
  \label{eq:app-channel-Y}
  &\text{Diffusion channel:} &\qquad
  Y &\;=\; X + \hat{\eta}_\tau, \qquad
  \hat{\eta}_\tau \;\sim\; \mathcal{CN}(0,\,\sigma^2_\tau), \\[4pt]
  \label{eq:app-channel-Z}
  &\text{NO channel:} &\qquad
  Z &\;=\; H\,X + \hat{\eta}_{\mathrm{NO}},
\end{alignat}
where $\hat{\eta}_\tau$ is Gaussian (by construction of the forward diffusion) and $\hat{\eta}_{\mathrm{NO}}$ has zero mean and variance $\sigma^2_{\mathrm{NO}}$ but is otherwise unrestricted in distributional form.
Both noise terms are mutually independent and independent of $X$.
The only assumptions on $X$ are $\mathbb{E}[X] = 0$ and
$\mathbb{E}[|X|^2] = P$ where $P = P_{\mathbf{u}}(k)$.

\begin{proof}
The conditional density of $Z$ given $Y$ is obtained by marginalizing
over $X$:
\begin{equation}
  \label{eq:app-marginal-integral}
  p(Z \mid Y)
  \;=\;
  \int p(Z \mid X)\,p(X \mid Y)\,dX,
\end{equation}
where $p(Z \mid X) = \mathcal{CN} (Z;\,HX,\,\sigma^2_{\mathrm{NO}})$ does not depend on $Y$ (by conditional independence $Z \perp Y \mid X$).

\medskip\noindent
\emph{Step (a): Differentiate under the integral.}
Under the usual regularity assumptions\citep{billingsley1995probability}, one can take the derivative $\nabla_{Y^*}$ of both sides:
\begin{equation}
  \nabla_{Y^*}\,p(Z \mid Y)
  \;=\;
  \int p(Z \mid X)\;\nabla_{Y^*}\,p(X \mid Y)\;dX.
\end{equation}
Rewriting $\nabla_{Y^*} p = p\,\nabla_{Y^*}\log p$:
\begin{equation}
  \nabla_{Y^*}\,p(Z \mid Y)
  \;=\;
  \int p(Z \mid X)\,p(X \mid Y)\;
  \nabla_{Y^*}\log p(X \mid Y)\;dX.
\end{equation}

\medskip\noindent
\emph{Step (b): Compute $\nabla_{Y^*}\log p(X \mid Y)$.}
By Bayes' rule, $p(X \mid Y) \propto p(Y \mid X)\,p_X(X)$,
where $p(Y \mid X) = \mathcal{CN}(Y;\,X,\,\sigma^2_\tau)$. 
The gradient reads:
\begin{equation}
  \nabla_{Y^*}\log p(X \mid Y)
  \;=\;
  \nabla_{Y^*}\log p(Y \mid X)
  \;-\;
  \nabla_{Y^*}\log p_Y(Y).
\end{equation}
The Gaussian channel gives $\nabla_{Y^*}\log p(Y \mid X) = (X - Y)/\sigma^2_\tau$, and the complex-value Tweedie formula provides the score of the marginal over $Y$, namely $\nabla_{Y^*}\log p_Y(Y) = (\mathbb{E}[X \mid Y] - Y)/\sigma^2_\tau$.
Finally:
\begin{equation}
  \label{eq:app-score-pXY}
  \nabla_{Y^*}\log p(X \mid Y)
  \;=\;
  \frac{X - \mathbb{E}[X \mid Y]}{\sigma^2_\tau}.
\end{equation}

\medskip\noindent
\emph{Step (c): Assemble.}
Substituting~\eqref{eq:app-score-pXY} and dividing both sides by
$p(Z \mid Y)$:
\begin{equation}
  \nabla_{Y^*}\log p(Z \mid Y)
  \;=\;
  \frac{1}{\sigma^2_\tau}
  \int \frac{p(Z \mid X)\,p(X \mid Y)}{p(Z \mid Y)}
  \;\bigl(X - \mathbb{E}[X \mid Y]\bigr)\;dX.
\end{equation}
By Bayes' rule:
$p(Z \mid X)\,p(X \mid Y) / p(Z \mid Y) = p(X \mid Y, Z)$,
so the integral gives $\mathbb{E}[X \mid Y, Z] - \mathbb{E}[X \mid Y]$, completing the proof.
\end{proof}

\subsection{Score error decomposition}
\label{app:score-error}
This section derives the score error
decomposition~\eqref{eq:score-error-main} stated in
Sec.~\ref{sec:score-analysis} and establishes the
distribution-free bounds on each term.

Our Gaussian approximation replaces both conditional expectations in~\eqref{eq:exact-score-main} by their LMMSE counterparts.

\paragraph{Linear estimates.}
From Appendix~\ref{app:lmmse}, the LMMSE estimate of $X$ from $Y$
alone is $\mathbb{E}_L[X \mid Y] = \alpha\,Y$ with error
$\tilde{X} \coloneqq X - \alpha Y$ of variance $\sigma_L^2$.

For the joint estimate from $(Y, Z)$, we apply a sequential LMMSE
update: the refined estimate takes the form $\mathbb{E}_L[X \mid Y, Z] = \alpha Y + \beta\,r$, where
\begin{equation}
  r \;\coloneqq\; Z - H\alpha Y
\end{equation}
is the \emph{innovation}, the component of $Z$ not predicted by the first-stage estimate.
Substituting $Z = HX + \hat{\eta}_{\mathrm{NO}}$:
\begin{equation}
  r \;=\; H(X - \alpha Y) + \hat{\eta}_{\mathrm{NO}}
  \;=\; H\tilde{X} + \hat{\eta}_{\mathrm{NO}}.
\end{equation}
By the orthogonality principle, the optimal coefficient $\beta$
satisfies $\beta = \mathrm{Cov}[\tilde{X},\,r]\,/\,\mathrm{Var}[r]$. For the numerator, since $\hat{\eta}_{\mathrm{NO}}$ is uncorrelated with $\tilde{X}$:
\begin{equation}
  \mathrm{Cov}[\tilde{X},\,r]
  \;=\;
  \mathrm{Cov}[\tilde{X},\,H\tilde{X}]
  \;=\;
  \mathbb{E}[\tilde{X}\,(H\tilde{X})^*]
  \;=\;
  H^*\,\sigma_L^2.
\end{equation}
For the denominator, by independence of $\hat{\eta}_{\mathrm{NO}}$ from $\hat{\eta}_\tau$ and orthogonality of $\hat{\eta}_{\mathrm{NO}}$ to $X$ 
(for $\hat{\eta}_{\mathrm{NO}} \perp \tilde X$ at second order, since $\tilde X = X - \alpha Y$ is a linear combination of $X$ and $\hat{\eta}_\tau$):
\begin{equation}
  \mathrm{Var}[r] \;=\; |H|^2\,\sigma_L^2 + \sigma^2_{\mathrm{NO}} 
  \;\eqqcolon\; \lambda_\tau.
\end{equation}
The LMMSE update is therefore
\begin{equation}
  \label{eq:app-lmmse-YZ}
  \mathbb{E}_L[X \mid Y, Z]
  \;=\;
  \alpha Y
  \;+\;
  \frac{H^*\,\sigma_L^2}{\lambda_\tau}\;r.
\end{equation}
The correction is large when the innovation is informative about the first-stage error (strong correlation between $\tilde{X}$ and $r$) and small when the innovation is dominated by the NO residual noise ($\lambda_\tau$ large).

The joint LMMSE error follows from the variance reduction formula:
\begin{equation}
  \label{eq:app-Sigma-LYZ}
  \sigma_{L,YZ}^2
  \;=\;
  \sigma_L^2
  - \frac{|\mathrm{Cov}[\tilde{X},\,r]|^2}{\mathrm{Var}[r]}
  \;=\;
  \sigma_L^2 - \frac{|H|^2\,\sigma_L^4}{\lambda_\tau}
  \;=\;
  \frac{\sigma_L^2\,\sigma^2_{\mathrm{NO}}}{\lambda_\tau}.
\end{equation}
Substituting the linear estimates into~\eqref{eq:exact-score-main}:
\begin{equation}
  \label{eq:app-gaussian-score}
  s^{\mathrm{approx}}
  \;=\;
  \frac{\mathbb{E}_L[X \mid Y, Z]
    - \mathbb{E}_L[X \mid Y]}{\sigma^2_\tau}
  \;=\;
  \frac{H^*\,\alpha}{\lambda_\tau}\;r.
\end{equation}

\paragraph{Score error.}
Subtracting from the exact score:
\begin{equation}
  \label{eq:app-score-error}
  \epsilon
  \;\coloneqq\;
  s^{\mathrm{exact}} - s^{\mathrm{approx}}
  \;=\;
  \frac{1}{\sigma^2_\tau}\Bigl[
  \underbrace{
  \bigl(\mathbb{E}[X \mid Y, Z]
    - \mathbb{E}_L[X \mid Y, Z]\bigr)
  }_{\eqqcolon\;\delta_{\mathrm{post}}}
  \;-\;
  \underbrace{
  \bigl(\mathbb{E}[X \mid Y] - \alpha Y\bigr)
  }_{\eqqcolon\;\delta_{\mathrm{prior}}}
  \Bigr].
\end{equation}
Both $\delta_{\mathrm{prior}}$ and $\delta_{\mathrm{post}}$ are \textbf{MMSE--LMMSE gaps}: the difference between the optimal nonlinear estimator and the optimal linear estimator of $X$.
If $X$ were Gaussian \emph{and} $\hat{\eta}_{\mathrm{NO}}$ were Gaussian, both gaps would vanish identically and the approximation would be exact. The score error is therefore driven by the non-Gaussianity of both the clean Fourier coefficients and the NO residual.

By the Pythagorean theorem in $L^2$:
\begin{align}
  \label{eq:app-Delta-prior-bound}
  \Delta_{\mathrm{prior}}
  &\;\coloneqq\;
  \mathbb{E}[|\delta_{\mathrm{prior}}|^2]
  \;=\; \sigma_L^2 - \mathrm{mmse}(X \mid Y)
  \;\leq\; \sigma_L^2, \\[4pt]
  \label{eq:app-Delta-post-bound}
  \Delta_{\mathrm{post}}
  &\;\coloneqq\;
  \mathbb{E}[|\delta_{\mathrm{post}}|^2]
  \;=\; \sigma_{L,YZ}^2 - \mathrm{mmse}(X \mid Y, Z)
  \;\leq\; \sigma_{L,YZ}^2.
\end{align}

\subsection{Regime analysis: detailed bounds}
\label{app:regime-details}

This section provides the asymptotic expansions and
explicit bounds supporting the regime summary in
Sec.~\ref{sec:regime-analysis}.
We use the dimensionless parameters $\nu$, $\gamma$,
and $\zeta$ defined therein, and write
$h \coloneqq |H|^2$ and $P \coloneqq P_{\mathbf{u}}(k)$
throughout.

\paragraph{Derived quantities in terms of $(\nu, \gamma)$.}
For reference:
\begin{equation}
  \renewcommand{\arraystretch}{1.6}
  \begin{array}{r@{\;=\;}l@{\qquad}r@{\;=\;}l}
  \alpha & \dfrac{1}{1+\nu},
  & \sigma_L^2 & \dfrac{P\,\nu}{1+\nu}, \\[6pt]
  \lambda_\tau & hP\!\left(\gamma + \dfrac{\nu}{1+\nu}\right),
  & \sigma_{L,YZ}^2 & \dfrac{P\,\nu\,\gamma}
    {(1+\nu)\gamma + \nu}, \\[6pt]
  \zeta & \dfrac{\nu}{\gamma(1+\nu)}.
  \end{array}
\end{equation}

All bounds below use only the distribution-free
inequalities $\Delta_{\mathrm{prior}} \leq \sigma_L^2$
and $\Delta_{\mathrm{post}} \leq \sigma_{L,YZ}^2$
from~\eqref{eq:app-Delta-prior-bound}--\eqref{eq:app-Delta-post-bound}.

\subsubsection{Regime I ($\nu \gg 1$)}

Expanding for $\nu \gg 1$:
$\alpha = O(\nu^{-1})$,
$\sigma_L^2 \to P$,
$\lambda_\tau = hP(\gamma + 1) + O(\nu^{-1})$.

\emph{Score magnitude:}
$\mathbb{E}[|s^{\mathrm{approx}}|^2]
= h\alpha^2/\lambda_\tau = O(\nu^{-2}) \to 0$.

\emph{Score error:}
\begin{equation}
  \mathbb{E}[|\epsilon|^2]
  \;\leq\;
  \frac{2(\Delta_{\mathrm{post}} + \Delta_{\mathrm{prior}})}
       {\sigma^4_\tau}
  \;\leq\;
  \frac{4P}{\sigma^4_\tau}
  \;=\;
  \frac{4}{P\nu^2}
  \;\to\; 0.
\end{equation}

\subsubsection{Regime II ($\nu \ll 1$, $\zeta \ll 1$)}

Asymptotics:
$\alpha \to 1$,
$\lambda_\tau \to \sigma^2_{\mathrm{NO}}$,
$\mathbb{E}[|s^{\mathrm{approx}}|^2] = O(1/(P\gamma))$,
$\sigma_{L,YZ}^2 \approx \sigma_L^2$ when $\zeta \ll 1$.

\emph{Absolute score error:}
\begin{equation}
  \mathbb{E}[|\epsilon|^2]
  \;\leq\;
  \frac{2(\sigma_{L,YZ}^2 + \sigma_L^2)}{\sigma^4_\tau}
  \;\approx\;
  \frac{4\sigma_L^2}{\sigma^4_\tau}
  \;=\;
  \frac{4}{P\nu}.
\end{equation}

\emph{Relative score error:}
$\mathbb{E}[|\epsilon|^2] / \mathbb{E}[|s^{\mathrm{approx}}|^2]
= O(\gamma/\nu)$.

The bound is not parametrically small because the triangle
inequality discards the partial cancellation between
$\delta_{\mathrm{post}}$ and $\delta_{\mathrm{prior}}$ that
arises when $Z$ provides limited additional information beyond
$Y$; quantifying this cancellation would require higher-order
control of $\hat{\eta}_{\mathrm{NO}}$, which we do not assume.

\paragraph{Spectral profile preservation.}
The weight $\tilde{\lambda}_\tau(k)$ is determined by
the calibrated quantities $H(k)$,
$\sigma^2_{\mathrm{NO}}(k)$, and $P_{\mathbf{u}}(k)$.
Non-Gaussianity may rescale the per-mode score magnitude
but cannot redistribute guidance across modes; any such
rescaling is absorbed into
$\lambda_{\mathrm{NO}}$.

\subsubsection{Regime III ($\nu \ll 1$, $\zeta \gg 1$)}

Requires $\gamma \ll \nu \ll 1$.
Expanding: $\alpha \approx 1$,
$\sigma_L^2 \approx \sigma^2_\tau$,
$\lambda_\tau \approx h\sigma^2_\tau$.

\emph{Score magnitude:}
$\mathbb{E}[|s^{\mathrm{approx}}|^2]
\approx 1/\sigma^2_\tau$.

\emph{Posterior gap:}
From~\eqref{eq:app-Sigma-LYZ}:
\begin{equation}
  \sigma_{L,YZ}^2
  \;\approx\;
  \frac{\sigma^2_\tau \cdot \sigma^2_{\mathrm{NO}}}
       {h\sigma^2_\tau}
  \;=\;
  \frac{\sigma^2_{\mathrm{NO}}}{h}
  \;=\; \gamma\,P.
\end{equation}

\emph{Score error:}
\begin{equation}
  \mathbb{E}[|\epsilon|^2]
  \;\leq\;
  \frac{2(\gamma P + P\nu)}{\sigma^4_\tau}
  \;=\;
  \frac{2\gamma}{P\nu^2} + \frac{2}{P\nu}
  \;=\;
  O\!\left(\frac{1}{P\nu}\right),
\end{equation}
since $\gamma \ll \nu$ in this regime makes the second term dominant. The absolute error scaling thus coincides with Regime~II; the regimes are distinguished by the \emph{relative} error, which is $O(\gamma/\nu)$ in Regime~II but $O(1)$ here.

\emph{Relative error:}
\begin{equation}
  \frac{\mathbb{E}[|\epsilon|^2]}
       {\mathbb{E}[|s^{\mathrm{approx}}|^2]}
  \;\leq\;
  \frac{2\gamma}{\nu} + 2.
\end{equation}
Since $\gamma \ll \nu$ defines this regime, $2\gamma/\nu \ll 1$ and the relative error is bounded by a constant.
The key mechanism is that the accuracy of the NO ($\gamma \ll \nu$) constrains the posterior gap $\Delta_{\mathrm{post}} \leq \gamma P$ independently of $\sigma_\tau$: one cannot enter Regime~III without
simultaneously providing the bound that controls the error.

\subsubsection{Regime IV ($\nu \sim 1$)}

At $\nu = 1$: $\alpha = 1/2$,
$\sigma_L^2 = P/2$,
$\sigma_{L,YZ}^2 = P\gamma/(2\gamma + 1)$.

All quantities are finite and the distribution-free bounds yield computable constants depending on $\gamma$ and $h$.
The regime contributes a bounded amount to the integrated score error over the diffusion trajectory.

\subsubsection{Summary}
\label{app:regime-summary}

The analysis is distribution-free in all four regimes. In Regimes~I, III, and IV the relative score error is bounded by a parametrically small or $O(1)$ constant. In Regime~II the relative-error bound is $O(\gamma/\nu)$ and is therefore not controlled by the analysis alone; what \emph{is} controlled, distribution-freely, is the spectral profile of the guidance, since the moment-matching cannot redistribute weight across modes (only rescale per-mode magnitudes by a multiplicative factor absorbed into $\lambda_{\mathrm{NO}}$).

\section{NO likelihood score: gradient derivation}
\label{app:no-score-derivation}

This appendix provides the full derivation of the closed-form NO likelihood score~\eqref{eq:no-score} stated in the main text.

\subsection{Energy function}

Taking the negative logarithm of the moment-matched Gaussian
marginal~\eqref{eq:app-marginal-result} and dropping all the constants, one obtains the energy function
\begin{equation}
  \label{eq:app-energy}
  J(\mathbf{u}_\tau)
  \;=\;
  \sum_{c,k}
  \frac{|\mathbf{r}(c,k)|^2}{\lambda_\tau(c,k)},
\end{equation}
with the spectral residual
\begin{equation}
  \label{eq:app-residual}
  \mathbf{r}(c,k)
  \;\coloneqq\;
  \mathcal{F}(\mathbf{\uNO})(c,k)
  \;-\;
  H(c,k)\,\alpha(c,k)\,\mathcal{F}(\mathbf{u}_\tau)(c,k).
\end{equation}
Note that this residual is computed against the marginalized mean
$H(k)\,\alpha(k)\,\mathcal{F}(\mathbf{u}_\tau)(k)$
from~\eqref{eq:app-marginal-result}; the factor $\alpha(k)$ enters through the LMMSE model, not through any approximation of the denoiser Jacobian.

In compact notation, $J = \mathbf{r}^\dagger \lambda_\tau^{-1}\mathbf{r}$, where $\lambda_\tau^{-1}$ is diagonal in the Fourier basis with entries $1/\lambda_\tau(c,k)$.

\subsection{Gradient computation}

A perturbation $d\mathbf{u}_\tau$ induces
$d\mathbf{r}(c,k) = -H(c,k)\,\alpha(c,k)\,\mathcal{F}(d\mathbf{u}_\tau)(c,k)$.
The differential of the quadratic form is
\begin{equation}
  dJ
  \;=\;
  2\,\mathrm{Re}\!\left[
    \mathbf{r}^\dagger \lambda_\tau^{-1}\,d\mathbf{r}
  \right]
  \;=\;
  -2\,\mathrm{Re}\!\left[
    \mathbf{r}^\dagger \lambda_\tau^{-1}
    \bigl((H\,\alpha) \odot \mathcal{F}(d\mathbf{u}_\tau)\bigr)
  \right],
\end{equation}
where $(H\,\alpha)(c,k) \coloneqq H(c,k)\,\alpha(c,k)$.

Using the adjoint identity
$\langle \mathbf{a},\,\mathcal{F}(\mathbf{b})\rangle
= \langle \mathcal{F}^\dagger(\mathbf{a}),\,\mathbf{b}\rangle$:
\begin{equation}
  dJ
  \;=\;
  -2\,\mathrm{Re}\!\left[
    \left(
      \mathcal{F}^\dagger\!\bigl(
        H^*\alpha \odot \lambda_\tau^{-1}\mathbf{r}
      \bigr)
    \right)^\dagger d\mathbf{u}_\tau
  \right].
\end{equation}
By the standard identification $dJ = (\nabla_{\mathbf{u}_\tau}J)^\top d\mathbf{u}_\tau$, and noting that for real-valued $\mathbf{u}_\tau$ the inverse transform of a Hermitian-symmetric spectrum is real (so $\mathrm{Re}[\cdot]$ can be dropped):
\begin{equation}
  \nabla_{\mathbf{u}_\tau}J
  \;=\;
  -2\,\mathcal{F}^{-1}\!\left(
    H^*\alpha \odot \lambda_\tau^{-1}\mathbf{r}
  \right).
\end{equation}

\subsection{Likelihood score}

The likelihood score is
$\nabla_{\mathbf{u}_\tau}\log p(\mathbf{\uNO} \mid \mathbf{u}_\tau)
= -\nabla_{\mathbf{u}_\tau}J$, giving:
\begin{equation}
  \label{eq:app-score-utau}
  \nabla_{\mathbf{u}_\tau}\log p(\mathbf{\uNO} \mid \mathbf{u}_\tau)
  \;=\;
  \mathcal{F}^{-1}\!\left(
    \tilde{\lambda}_\tau \odot \mathbf{r}
  \right),
\end{equation}
with the spectral weighting filter
\begin{equation}
  \label{eq:app-Lambda-tilde}
  \tilde{\lambda}_\tau(c,k)
  \;\coloneqq\;
  \frac{2\,H^*(c,k)\,\alpha(c,k)}{\lambda_\tau(c,k)}
  \;=\;
  \frac{
    2\,H^*(c,k)\,P_{\mathbf{u}}(c,k)
  }{
    \left(\sigma^2_\tau + P_{\mathbf{u}}(c,k)\right)
      \sigma^2_{\mathrm{NO}}(c,k)
    \;+\;
    \sigma^2_\tau\,|H(c,k)|^2\,P_{\mathbf{u}}(c,k)
  }.
\end{equation}

\paragraph{Remark.}
Since $H(k) \in \mathbb{R}$ (Sec.~\ref{sec:no-guided-dps}),
$H^*(k) = H(k)$ throughout; we retain the conjugate notation for
generality.

\section{Sampling algorithm}
\label{app:algorithm}

Algorithm~\ref{alg:dps-no} details the full posterior sampling
procedure combining the unconditional diffusion prior, sparse sensor guidance, and spectral NO guidance as described in
Sec.~\ref{sec:no-guided-dps}.

\begin{algorithm}[h]
\caption{Neural operator-guided diffusion posterior sampling}
\label{alg:dps-no}
\begin{algorithmic}[1]
\REQUIRE Denoiser $D_\theta$; diffusion schedule
  $\{(\sigma_{\tau_i}, s_{\tau_i})\}_{i=0}^{N_{\mathrm{steps}}}$;
  spectral tables $H(c,k)$, $\sigma^2_{\mathrm{NO}}(c,k)$,
  $P_{\mathbf{u}}(c,k)$;
  NO prediction $\mathbf{\uNO}$;
  sparse observations $\mathbf{y}$;
  sensor operator $\mathcal{M}_{\mathcal{S}}$;
  hyperparameters $\lambda_{\mathrm{s}}$, $\lambda_{\mathrm{NO}}$
\ENSURE Posterior sample $\mathbf{u}_0$

\STATE \textbf{Precompute (once):}
\STATE \quad $\hat{\mathbf{u}}_{\mathrm{NO}} \leftarrow
  \mathcal{F}(\mathbf{\uNO})$
  \hfill $\triangleright$ FFT of NO prediction

\medskip
\STATE \textbf{Initialise:}
\STATE \quad $\mathbf{u}_{\tau_0} \sim
  \mathcal{N}(\mathbf{0},\,
  \sigma^2_{\tau_0}\,s^2_{\tau_0}\,I)$

\medskip
\FOR{$i = 0, \dots, N_{\mathrm{steps}}-1$}
  \STATE $\sigma \leftarrow \sigma_{\tau_i}$, \quad
         $s \leftarrow s_{\tau_i}$, \quad
         $\Delta\tau \leftarrow \tau_{i+1} - \tau_i$

  \medskip
  \STATE \textit{// (i) Prior score via denoiser}
  \STATE $\mathbf{u}_{\tau_i}^{\mathrm{req}} \leftarrow
    \mathbf{u}_{\tau_i}$ with gradients enabled
  \STATE $\bar{\mathbf{u}}_0 \leftarrow
    D_\theta(\mathbf{u}_{\tau_i}^{\mathrm{req}} / s,\;\sigma)$
    \hfill $\triangleright$ denoised estimate
  \STATE $d_{\mathrm{prior}} \leftarrow
    \left(\dfrac{\dot{\sigma}}{\sigma}
    + \dfrac{\dot{s}}{s}\right)\mathbf{u}_{\tau_i}
    - \dfrac{\dot{\sigma}}{\sigma}\,s\,\bar{\mathbf{u}}_0$

  \medskip
  \STATE \textit{// (ii) NO guidance (spectral, no backprop)}
  \STATE $\alpha(c,k) \leftarrow
    P_{\mathbf{u}}(c,k)\,/\,
    (\sigma^2 + P_{\mathbf{u}}(c,k))$
  \STATE $\lambda_\tau(c,k) \leftarrow
    \sigma^2_{\mathrm{NO}}(c,k)
    + |H(c,k)|^2\,\sigma^2\,\alpha(c,k)$
  \STATE $\tilde{\lambda}_\tau(c,k) \leftarrow
    2\,H^*(c,k)\,\alpha(c,k)\,/\,\lambda_\tau(c,k)$
  \STATE $\mathbf{r}(c,k) \leftarrow
    \hat{\mathbf{u}}_{\mathrm{NO}}(c,k)
    - H(c,k)\,\alpha(c,k)\,
    \mathcal{F}(\mathbf{u}_{\tau_i}/s)(c,k)$
  \STATE $\mathbf{g}_{\mathrm{NO}} \leftarrow
    \mathcal{F}^{-1}\!\left(
    \tilde{\lambda}_\tau \odot \mathbf{r}\right)$
  \STATE $c_\tau \leftarrow
    s^2\,\sigma\,\dot{\sigma}$
  \STATE $d_{\mathrm{NO}} \leftarrow
    c_\tau\,\lambda_{\mathrm{NO}}\,
    \mathbf{g}_{\mathrm{NO}}\,/\,s$

  \medskip
  \STATE \textit{// (iii) Sensor guidance (DPS, with backprop)}
  \STATE $\mathbf{g}_{\mathrm{sensor}} \leftarrow
    \mathcal{M}_{\mathcal{S}}^\dagger\!\left(
    \mathbf{y} - \mathcal{M}_{\mathcal{S}}(s\,\bar{\mathbf{u}}_0)
    \right)$
  \STATE $r_{\mathrm{obs}} \leftarrow
    \|\mathcal{M}_{\mathcal{S}}(s\,\bar{\mathbf{u}}_0)
    - \mathbf{y}\|$
  \STATE $\mathbf{g}_{\mathrm{vjp}} \leftarrow
    \mathrm{VJP}\!\left(
    \mathbf{u}_{\tau_i}^{\mathrm{req}} \mapsto
    s\,\bar{\mathbf{u}}_0;\;
    \mathbf{g}_{\mathrm{sensor}}\right)$
    \hfill $\triangleright$ backprop through denoiser
  \STATE $d_{\mathrm{sensor}} \leftarrow
    \lambda_{\mathrm{s}}\,\mathbf{g}_{\mathrm{vjp}}\,/\,
    r_{\mathrm{obs}}$

  \medskip
  \STATE \textit{// Euler step}
  \STATE $\mathbf{u}_{\tau_{i+1}} \leftarrow
    \mathbf{u}_{\tau_i}
    + \left(d_{\mathrm{prior}}
    - d_{\mathrm{NO}}
    - d_{\mathrm{sensor}}\right)\,\Delta\tau$
\ENDFOR

\medskip
\STATE \textit{// Final denoising}
\STATE $\mathbf{u}_0 \leftarrow
  D_\theta(\mathbf{u}_{\tau_{N_{\mathrm{steps}}}} /
  s_{\tau_{N_{\mathrm{steps}}}},\;
  \sigma_{\tau_{N_{\mathrm{steps}}}})$

\RETURN $\mathbf{u}_0$
\end{algorithmic}
\end{algorithm}

\paragraph{Implementation notes.}
Under the VE schedule ($s_\tau = 1$ throughout), the scaling by $s$ and $1/s$ in the algorithm reduces to identity and can be omitted.
The spectral quantities $\alpha(c,k)$, $\lambda_\tau(c,k)$, and
$\tilde{\lambda}_\tau(c,k)$ depend on $\sigma_{\tau_i}$ and must
be recomputed at each step; since they involve only elementwise
operations on the precomputed lookup tables, this cost is negligible.
The FFT of the NO prediction $\hat{\mathbf{u}}_{\mathrm{NO}}$ is
computed once and reused across all steps.
The dominant per-step cost is the denoiser forward pass (shared by terms~(i) and~(iii)) and the VJP backward pass (term~(iii) only).

\section{Experimental details}
\label{app:experimental}

\subsection{Denoiser architecture and training}
\label{app:denoiser-training}

The unconditional diffusion prior uses the GenCFD
architecture \citep{molinaro2024generative}, a 3D preconditioned denoiser based on a U-Net backbone with channel widths $(64, 128, 256)$, downsample ratios $(2, 2, 2)$, $4$ attention blocks with $8$ heads, and noise
embedding dimension $128$.
The model is trained under the variance-exploding (VE) diffusion
scheme with an exponential noise schedule spanning
$\sigma_{\min} = 0.002$ to $\sigma_{\max} = 80$, using EDM-style
weighting \citep{karras2022elucidating} and log-uniform noise sampling.

Training uses Adam with peak learning rate $3 \times 10^{-4}$ and
weight decay $0.01$, run for $430{,}000$ steps with a total batch
size of $32$ ($4$ per GPU $\times$ $8$ NVIDIA A100 80\,GB GPUs).
An exponential moving average of parameters with decay $0.999$ is
maintained and used for all inference.

\paragraph{Data normalization.}
The training targets are surface velocity wavefields
$\mathbf{u} \in \mathbb{R}^{3 \times 32 \times 32 \times 320}$.
Each sample is first scaled by a physics-based normalization constant that accounts for source--receiver distance and local S-wave velocity (following the MIFNO convention of \citep{lehmann2025multiple}),
then z-score normalized per channel using global statistics
(mean and standard deviation) computed over the $27{,}000$-sample
training set.
The same two-stage normalization is applied to MIFNO predictions
at inference time to ensure both the diffusion prior and the NO
guidance operate in a common normalized space.

\subsection{MIFNO surrogate}
\label{app:mifno-details}

We use the pretrained MIFNO model of~\citet{lehmann2025multiple},
frozen throughout our pipeline and used without modification.
The model uses $L=16$ factorized Fourier layers with channel width $d_v=16$, retaining $16$ Fourier modes along each spatial axis and $32$ along the temporal/depth axis (except the first layer, which uses $16$).
The source branch consists of a two-layer perceptron ($128$ hidden
units) followed by two 2D convolutional layers.
The model totals approximately $\mathrm{3.4}$\,M parameters and was
trained for $600$ epochs on the $27{,}000$-sample training split.

\subsection{Spectral calibration}
\label{app:spectral-calibration}

The three spectral quantities, $H(c,k)$,
$\sigma^2_{\mathrm{NO}}(c,k)$, and $P_{\mathbf{u}}(c,k)$, are
estimated from $N = 2{,}000$ paired ground-truth/MIFNO samples on a dedicated calibration split, disjoint from both the training and test sets.
Estimator formulas are given in Appendix~\ref{app:spectral-estimation}.
The resulting lookup tables are stored as a single
\texttt{.pt} checkpoint and loaded at inference time.

\subsection{Guidance hyperparameters}
\label{app:hyperparams}

The posterior sampling ODE (Algorithm~\ref{alg:dps-no}) involves two scalar hyperparameters: $\lambda_{\mathrm{s}}$ for the sensor guidance and $\lambda_{\mathrm{NO}}$ for the NO guidance.
The DPS sensor term uses the step-size convention
$\lambda_{\mathrm{s}} / \|\mathcal{M}_{\mathcal{S}}(\bar{\mathbf{u}}_0) - \mathbf{y}\|$, which absorbs the observation noise variance $\sigma_y^2$.
Table~\ref{tab:hyperparams} reports the values used for each
method and sensor density.

\begin{table}[h]
\centering
\small
\caption{Guidance hyperparameters for each method and sensor density. $\lambda_{\mathrm{s}}$: sensor guidance step size; $\lambda_{\mathrm{NO}}$: NO guidance weight.
A dash indicates the term is absent.}
\label{tab:hyperparams}
\begin{tabular}{@{}lcccc@{}}
\toprule
& \multicolumn{2}{c}{$\rho = 5\%$}
& \multicolumn{2}{c}{$\rho = 2\%$} \\
\cmidrule(lr){2-3} \cmidrule(lr){4-5}
& $\lambda_{\mathrm{s}}$
& $\lambda_{\mathrm{NO}}$
& $\lambda_{\mathrm{s}}$
& $\lambda_{\mathrm{NO}}$ \\
\midrule
DPS
  & $23{,}000$ & ---
  & $23{,}000$ & --- \\[2pt]
DPS + NO (iso)
  & $23{,}000$ & $10{,}000$
  & $23{,}000$ & $10{,}000$ \\[2pt]
\FreqNODPS\ ($\alpha\!=\!1$)
  & $23{,}000$ & $0.1$
  & $10{,}000$ & $0.1$ \\[2pt]
\FreqNODPS
  & $23{,}000$ & $0.35$
  & $23{,}000$ & $0.35$ \\
\bottomrule
\end{tabular}
\end{table}

The hyperparameters $\lambda_s$ and $\lambda_{\mathrm{NO}}$ were selected on the same $2{,}000$ sample split used for spectral calibration (App.~\ref{app:spectral-calibration}), disjoint from both the $27{,}000$ sample training set and the $1{,}000$ sample test set used for all reported results. The sensitivity analysis in Appendix~\ref{app:sensitivity} reports the response of the test set metrics to $\lambda_{\mathrm{NO}}$ post hoc, to characterize the curvature of the objective around the selected operating point and verify that it coincides with the zero-crossing of rFFT\textsubscript{high}. It is not used for selection.

\paragraph{Remark.}
The large numerical value of $\lambda_{\mathrm{s}}$ reflects the DPS residual-norm normalization: the measurement residual $\|\mathcal{M}_{\mathcal{S}} \bar{\mathbf{u}}_0) - \mathbf{y}\|$ is small in z-score normalized space, so the step size must compensate.
The NO weight differs by four orders of magnitude between DPS\,+\,NO\,(iso) ($\lambda_{\mathrm{NO}} = 10{,}000$) and \FreqNODPS\ ($\lambda_{\mathrm{NO}} = 0.35$) because the two methods route the guidance differently: in DPS\,+\,NO\,(iso), $\lambda_{\mathrm{NO}}$ multiplies the VJP through the denoiser (analogous to $\lambda_{\mathrm{s}}$), whereas in \FreqNODPS\ it scales the closed-form spectral score whose magnitude is set by the calibrated $\tilde{\lambda}_\tau(k)$.
The near-unity value of $\lambda_{\mathrm{NO}}$ for \FreqNODPS\ is consistent with the principled spectral calibration: the likelihood score is already correctly scaled, and $\lambda_{\mathrm{NO}}$ serves only as a fine adjustment.

\subsection{Sensor mask generation}
\label{app:mask-generation}

For each test sample, the sensor subset
$\mathcal{S} \subset \{1,\dots,N_x\} \times \{1,\dots,N_y\}$ is generated by drawing $|\mathcal{S}| = \lfloor \rho \cdot N_x N_y \rfloor$ locations uniformly at random without replacement. The random seed is fixed per sample index to ensure that all
methods are evaluated on identical sensor  configurations.
At $\rho = 5\%$ this yields $|\mathcal{S}| = 51$ sensors; at $\rho = 2\%$, $|\mathcal{S}| = 20$.
No spatial regularity or optimization of the sensor layout is imposed.

\subsection{Sampling configuration}
\label{app:sampling-config}

Posterior samples are generated by solving the probability-flow ODE using the explicit Euler integrator with $64$ time steps following the EDM noise decay schedule \citep{karras2022elucidating}.
A final denoising step is applied at the terminal noise level.

We observed that increasing the number of steps beyond $64$ yields negligible improvement in reconstruction quality, while reducing below ${\sim}48$ steps leads to noticeable degradation.
All reported results use the ODE formulation; the SDE variant was also implemented but not used in the final experiments.

\subsection{Pointwise accuracy}

Following \citet{lehmann2025multiple}, the pointwise metrics are
computed per sensor location $(x,y)$ over the temporal axis,
then averaged over all sensors, channels, and samples.

\paragraph{Relative Root Mean Squared Error (rRMSE).}
\begin{equation}
  \mathrm{rRMSE}(x,y)
  \;\coloneqq\;
  \sqrt{
    \frac{1}{N_t}\sum_{k=1}^{N_t}
    \frac{
      \bigl(\hat{u}(x,y,t_k) - u(x,y,t_k)\bigr)^2
    }{
      u(x,y,t_k)^2 + \epsilon^2
    }
  },
\end{equation}
with $\epsilon = 0.01$.

\paragraph{Relative Mean Absolute Error (rMAE).}
\begin{equation}
  \mathrm{rMAE}(x,y)
  \;\coloneqq\;
  \frac{1}{N_t}\sum_{k=1}^{N_t}
  \frac{
    \bigl|\hat{u}(x,y,t_k) - u(x,y,t_k)\bigr|
  }{
    \bigl|u(x,y,t_k)\bigr| + \epsilon
  }.
\end{equation}
Both metrics are averaged over all sensors $(x,y) \in \mathcal{G}$,
velocity components $c \in \{E,N,Z\}$, and test samples.

\subsection{Spectral fidelity}

Following~\citet{lehmann2025multiple}, the frequency content is
assessed via banded relative spectral biases.
For each sensor $(x,y)$, let $\mathcal{F}(u(x,y))(f)$ denote
the temporal Fourier transform at frequency $f$ ($\Delta t = 0.02$\,s, $f_{\mathrm{Nyq}} = 25$\,Hz).
Define the band-averaged spectral magnitude
\begin{equation}
  \overline{\mathcal{F}(u(x,y))}_{\mathcal{B}}
  \;\coloneqq\;
  \frac{1}{N_f}\sum_{f \in \mathcal{B}}
  \bigl|\mathcal{F}(u(x,y))(f)\bigr|,
\end{equation}
where $N_f$ is the number of frequency bins in band $\mathcal{B}$.
The banded relative FFT bias is then
\begin{equation}
  \mathrm{rFFT}_{\mathcal{B}}(x,y)
  \;\coloneqq\;
  \frac{
    \overline{\mathcal{F}(\hat{u}(x,y))}_{\mathcal{B}}
    \;-\;
    \overline{\mathcal{F}(u(x,y))}_{\mathcal{B}}
  }{
    \overline{\mathcal{F}(u(x,y))}_{\mathcal{B}}
  },
\end{equation}
averaged over all sensors, channels, and samples.
Three bands are reported:
$\mathrm{rFFT}_{\mathrm{low}}$ ($0$--$1$\,Hz),
$\mathrm{rFFT}_{\mathrm{mid}}$ ($1$--$2$\,Hz), and
$\mathrm{rFFT}_{\mathrm{high}}$ ($2$--$5$\,Hz).
Negative values indicate systematic spectral underestimation (attenuation); positive values indicate overestimation; zero indicates unbiased spectral reproduction.

\subsection{Significant duration}

The significant duration $D_{5\text{--}95}$ quantifies the time window containing the central $90\%$ of the seismic energy at each spatial location, based on the Arias intensity. For each grid point $(x,y)$, the cumulative Arias intensity is
\begin{equation}
  I_A(x,y,t)
  \;=\;
  \Delta t \sum_{k=1}^{\lfloor t/\Delta t \rfloor}
  \sum_{c=1}^{C} v_c(x,y,t_k)^2,
\end{equation}
and the bounding times $t_5$, $t_{95}$ are the earliest instants at which $I_A$ reaches $5\%$ and $95\%$ of the total $I_A(x,y,T)$, respectively, with linear interpolation between discrete time steps.
The significant duration is then $D_{5\text{--}95}(x,y) = t_{95}(x,y) - t_{5}(x,y)$.
Grid points with zero total energy are excluded.
We report the absolute error $|D_{5\text{--}95}^{\mathrm{pred}} - D_{5\text{--}95}^{\mathrm{true}}|$ averaged over all grid
locations $(x,y) \in \mathcal{G}$ and test samples.
This metric is sensitive to temporal energy distribution:
oversmoothed predictions spread energy over a wider window (overestimating $D_{5\text{--}95}$), while methods that miss late-arriving scattered phases produce shorter durations.

\section{Sensitivity to the NO guidance weight}
\label{app:sensitivity}

The closed-form NO score~\eqref{eq:no-score} is scaled by a single scalar hyperparameter $\lambda_{\mathrm{NO}}$ in the sampler (Algorithm~\ref{alg:dps-no}). The regime analysis of Section~\ref{sec:regime-analysis} predicts that, in the operating regime (Regime~II), the spectral \emph{profile} of the guidance is preserved regardless of distributional assumptions while per-mode magnitude rescaling is absorbed into $\lambda_{\mathrm{NO}}$.

We sweep $\lambda_{\mathrm{NO}}$ across two orders of magnitude on the $1{,}000$ test samples set at $\rho=5\%$, with all other hyperparameters fixed.
We emphasize that $\lambda_{\mathrm{NO}}=0.35$ was selected on the validation/calibration split (App.~\ref{app:hyperparams}) prior to any test set evaluation; the sweep below reports the test-set response post hoc to characterize sensitivity, not to select the operating point.

\begin{table}[h]
\centering
\small
\caption{Sensitivity of pointwise and spectral metrics to the NO guidance weight $\lambda_{\mathrm{NO}}$ at $\rho=5\%$. Calibrated operating point $\lambda_{\mathrm{NO}}=0.35$) shown in bold.
Mean over $1{,}000$ test samples.}
\label{tab:sensitivity}
\begin{tabular}{@{}lcccccc@{}}
\toprule
$\lambda_{\mathrm{NO}}$ & rMAE & rRMSE
  & rFFT\textsubscript{low} & rFFT\textsubscript{mid}
  & rFFT\textsubscript{high} \\
\midrule
$0.035$ & $0.108$ & $0.207$ & $-0.062$ & $-0.126$ & $-0.189$ \\
$0.10$  & $0.103$ & $0.198$ & $-0.028$ & $-0.075$ & $-0.122$ \\
$0.20$  & $0.101$ & $0.195$ & $\mathbf{-0.007}$ & $-0.043$ & $-0.063$ \\
$\mathbf{0.35}$ & $\mathbf{0.100}$ & $\mathbf{0.200}$
  & $+0.009$ & $\mathbf{-0.015}$ & $\mathbf{+0.002}$ \\
$0.70$  & $0.103$ & $0.222$ & $+0.035$ & $+0.027$ & $+0.121$ \\
$1.50$  & $0.115$ & $0.275$ & $+0.067$ & $+0.081$ & $+0.313$ \\
$3.50$  & $0.132$ & $0.325$ & $+0.085$ & $+0.105$ & $+0.457$ \\
\bottomrule
\end{tabular}
\end{table}

\paragraph{Pointwise stability.}
The relative MAE varies by less than $8\%$ over the range $\lambda_{\mathrm{NO}} \in [0.035, 0.70]$, a $20\times$ span. Only for $\lambda_{\mathrm{NO}} \geq 1.5$ does the surrogate guidance begin to dominate the prior, producing measurable degradation in both rMAE and rRMSE. The pointwise insensitivity to $\lambda_{\mathrm{NO}}$ within an order of magnitude of the calibrated value is a direct empirical realization of the Regime~II claim that non-Gaussianity can rescale the per-mode guidance magnitude but cannot redistribute it across modes: a
multiplicative correction is absorbed by a scalar without affecting the reconstruction's spatial structure.

\paragraph{Monotone spectral correction and natural operating point.} 
The banded spectral bias rFFT$_\mathcal{B}$ varies
monotonically with $\lambda_{\mathrm{NO}}$ in all three bands, passing from systematic under-correction at small $\lambda_{\mathrm{NO}}$ (rFFT\textsubscript{high} $= -0.189$ at
$\lambda_{\mathrm{NO}} = 0.035$, close to the MIFNO baseline of $-0.239$) to systematic over-correction at large $\lambda_{\mathrm{NO}}$ (rFFT\textsubscript{high} $= +0.457$ at $\lambda_{\mathrm{NO}} = 3.5$). The zero-crossing in the high band occurs essentially at the calibrated operating point $\lambda_{\mathrm{NO}} = 0.35$ (rFFT\textsubscript{high} $= +0.002$). The calibrated value, determined from the spectral statistics $H(k)$, $\sigma^2_{\mathrm{NO}}(k)$, and $P_{\mathbf{u}}(k)$ without reference to the test set, therefore sits at the principled zero-crossing of the spectral  correction, not at an empirically tuned local optimum.

\begin{figure}[h]
  \centering
  \includegraphics[width=\textwidth]{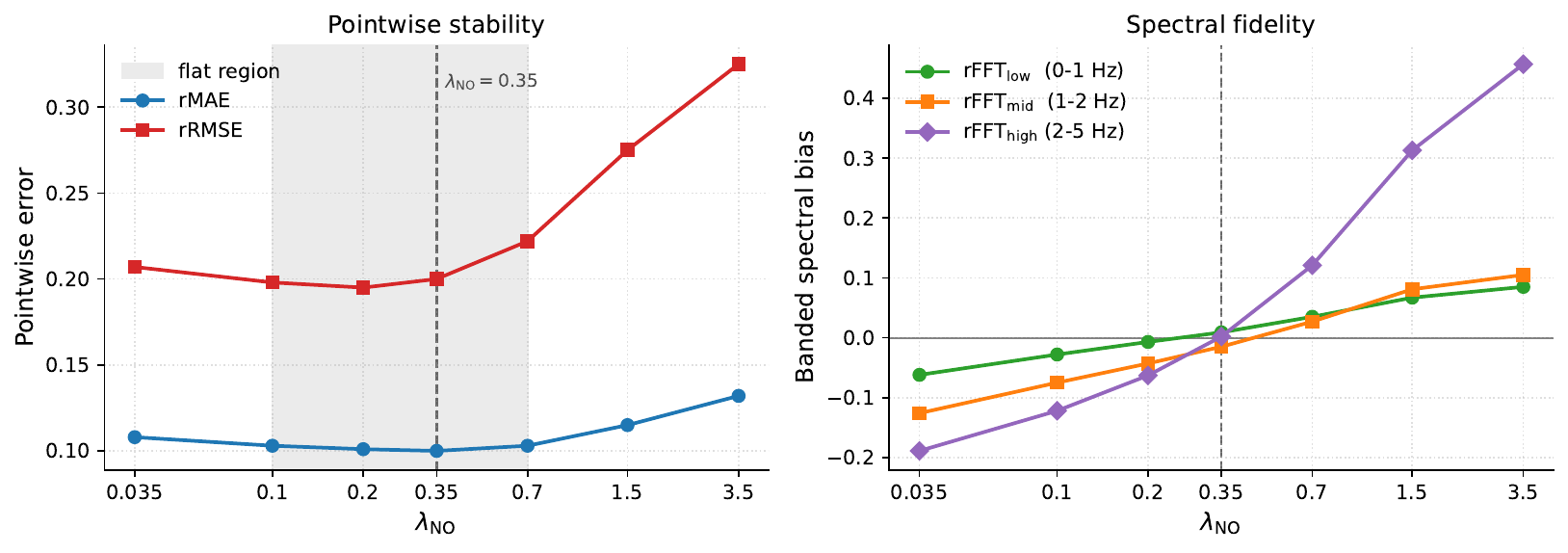}
  \caption{Sensitivity to $\lambda_{\mathrm{NO}}$ at $\rho = 5\%$.
    Vertical dashed line: calibrated operating point
    $\lambda_{\mathrm{NO}} = 0.35$.
    Left: pointwise metrics (rMAE, rRMSE) are flat across the
    shaded region $\lambda_{\mathrm{NO}} \in [0.1, 0.7]$.
    Right: banded spectral biases vary monotonically and cross
    zero near the calibrated value, with rFFT\textsubscript{high}
    spanning $-0.189$ to $+0.457$ across the sweep.}
  \label{fig:sensitivity}
\end{figure}



\end{document}